\documentclass{article}
\usepackage{lmodern}
\usepackage{xcolor}
\usepackage{arxiv}
\usepackage{authblk}
\usepackage{enumitem} %
\usepackage[titletoc]{appendix}
\usepackage{minitoc}
\usepackage{amsmath, amssymb, amsthm}
\usepackage{mdframed}
\usepackage{physics}
\usepackage[T1]{fontenc}
\usepackage{float} %
\usepackage{algorithm2e}
\RestyleAlgo{ruled}
\usepackage{graphicx}
\usepackage{wrapfig}
\makeatletter

\renewcommand{\@algocf@capt@plain}{above}%
\makeatother
\usepackage{subcaption}
\usepackage{hyperref}
\usepackage{xcolor}         %
\hypersetup{colorlinks=true,allcolors=blue}

\newcommand\blfootnote[1]{%
  \begingroup
  \renewcommand\thefootnote{}%
  \footnotetext{#1}%
  \endgroup
}
\newtheorem{theorem}{Theorem}

\newtheorem{corollary}{Corollary}

\newtheorem{lemma}{Lemma}
\newtheorem*{theorem*}{Theorem}
\newtheorem*{corollary*}{Corollary}

\usepackage[T1]{fontenc}
\def\bw{\bar{\mathbf{w}}}
\def\la{\langle

}

\usepackage{parskip}
\def\ra{\rangle}
\def\v{{\mathbf{v}}}

\def\x{{\mathbf{x}}}
\def\w{{\mathbf{w}}}
\def\m{{\mathbf{m}}}

\def\H{{\mathbf{H}}}

\def\M{{\mathbf{M}}}
\def\D{{\mathbf{D}}}
\def\Q{{\mathbf{Q}}}

\def\H{{\mathbf{H}}}
\def\I{{\mathbf{I}}}

\def\risk{\mathcal{R}}
\def\bias{\mathtt{bias}}
\def\variance{\mathtt{variance}}

\def\bm{\widetilde{\m}}
\def\vm{\overline{\m}}

\def\diag{\text{diag}}

\def\LLambda{{\mathbf{\Lambda}}}

\def\SSigma{\mathbf{\Sigma}}
\def\N{\mathcal{N}}
\def\tr{\text{Tr}}

\graphicspath{{../figures/}}
\graphicspath{{../.}}

\usepackage{amsmath,amsfonts,bm}

\def\eqref#1{equation~\ref{#1}}

\def\1{\mathbf{1}}

\def\ra{{\textnormal{a}}}

\def\vm{{\bm{m}}}

\DeclareMathAlphabet{\mathsfit}{\encodingdefault}{\sfdefault}{m}{sl}
\SetMathAlphabet{\mathsfit}{bold}{\encodingdefault}{\sfdefault}{bx}{n}

\newcommand{\E}{\mathbb{E}}

\newcommand{\R}{\mathbb{R}}

\def\bw{\bar{\mathbf{w}}}
\def\la{\langle}
\def\ra{\rangle}
\def\v{{\mathbf{v}}}

\def\x{{\mathbf{x}}}
\def\w{{\mathbf{w}}}
\def\m{{\mathbf{m}}}
\def\H{{\mathbf{H}}}

\def\M{{\mathbf{M}}}
\def\D{{\mathbf{D}}}
\def\Q{{\mathbf{Q}}}

\def\H{{\mathbf{H}}}

\def\risk{\mathcal{R}}
\def\bias{\mathtt{bias}}
\def\variance{\mathtt{variance}}

\def\bm{\widetilde{\m}}
\def\vm{\overline{\m}}

\def\LLambda{{\mathbf{\Lambda}}}

\def\SSigma{\mathbf{\Sigma}}
\def\N{\mathcal{N}}
\def\tr{\text{Tr}}
\def\I{\mathbf{I}}

\title{Anytime Pretraining:\\ Horizon-Free Learning-Rate Schedules with Weight Averaging}

\author[*,1,2]{Alexandru Meterez}
\author[*,1,2]{Pranav Ajit Nair}
\author[*,1,2]{Depen Morwani}
\author[1,2]{\\ Cengiz Pehlevan}
\author[1,2]{Sham Kakade}

\affil[1]{Harvard University}
\affil[2]{Kempner Institute at Harvard University}

\theoremstyle{plain}

\def\bw{\bar{\mathbf{w}}}
\def\la{\left\langle}
\def\ra{\right\rangle}
\def\x{{\mathbf{x}}}
\def\w{{\mathbf{w}}}
\def\m{{\mathbf{m}}}
\def\H{{\mathbf{H}}}
\def\Q{{\mathbf{Q}}}
\def\M{{\mathbf{M}}}
\def\I{{\mathbf{I}}}
\def\diag{{\text{diag}}}

\def\risk{\mathcal{R}}
\def\bias{\mathtt{bias}}
\def\variance{\mathtt{variance}}

\def\cov{\text{Cov}}

\def\bm{\widetilde{\m}}
\def\vm{\overline{\m}}
\def\E{\mathbb{E}}
\def\LLambda{{\mathbf{\Lambda}}}

\def\SSigma{\mathbf{\Sigma}}
\def\N{\mathcal{N}}
\def\R{\mathbb{R}}
\begin{document}
\setlength{\parindent}{0pt}

\maketitle
\begin{abstract}
    Large language models are increasingly trained in continual or open-ended settings, where the total training horizon is not known in advance. Despite this, most existing pretraining recipes are not anytime: they rely on horizon-dependent learning rate schedules and extensive tuning under a fixed compute budget. In this work, we provide a theoretical analysis demonstrating the existence of anytime learning schedules for overparameterized linear regression, and we highlight the central role of weight averaging—also known as model merging—in achieving the minimax convergence rates of stochastic gradient descent. We show that these anytime schedules polynomially decay with time, with the decay rate determined by the source and capacity conditions of the problem. Empirically, we evaluate 150M and 300M parameter language models trained for up to 32$\times$ Chinchilla scale, comparing constant and $1/\sqrt{t}$ schedules with weight averaging against a well-tuned cosine schedule. Across the full training range, the anytime schedules achieve comparable final loss to cosine decay. Taken together, our results suggest that weight averaging combined with simple, horizon-free step sizes offers a practical and effective anytime alternative to cosine learning rate schedules for large language model pretraining.
\end{abstract}

\section{Introduction}
\blfootnote{\hspace{-6mm}${}^*$: Equal contribution. \\ Correspondence to:  \texttt{ameterez@g.harvard.edu, pranavnair@g.harvard.edu} \\ Code is available at \url{https://github.com/alexandrumeterez/anytime-training}.}

\label{sec:introduction}
Since its introduction in the seminal paper by~\citet{loshchilov2016sgdr}, cosine decay has become the standard learning rate scheduler for large language model (LLM) pretraining, being used by practitioners for training frontier models across a wide range of scales and architectures~\citep{dubey2024llama,olmo2025olmo,liu2024deepseek}. Despite its widespread adoption, cosine decay is a horizon dependent scheduler, since it requires knowing the length of the training run ahead of time, making it ill suited for continual learning setups in which data is continuously received and mixed into the pretraining run. 

Recent work~\citep{hu2024minicpm,wen2024understanding} has proposed the warmup-stable-decay (WSD) schedule as a nearly-anytime alternative to cosine decay, achieving competitive performance~\citep{team2025kimi}. WSD consists of $3$ phases, namely a linear learning rate warm-up phase up to a maximum learning rate, then keeping the learning rate constant until $90\%$ of the training length, followed by decaying the learning rate for the rest of the $10\%$.
\begin{wrapfigure}{r}{0.50\textwidth}
  \centering
  \includegraphics[width=0.50\textwidth]{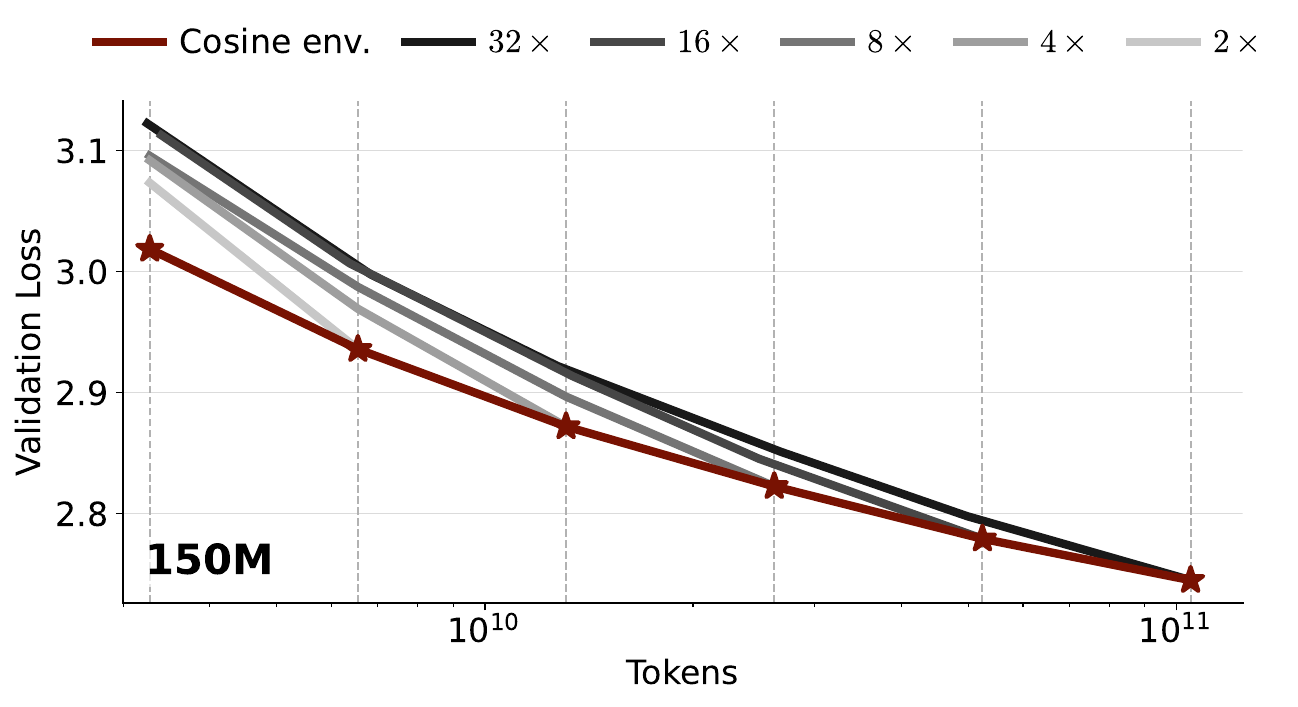}
  \caption{150M model: \textbf{Cosine schedules do not transfer across horizons.} The \emph{cosine envelope} (red) is formed by independently tuning a horizon-aware cosine schedule for each terminal compute budget ($1\times$--$32\times$ Chinchilla) and taking the best validation loss at that horizon. Gray curves show the same cosine schedule evaluated at intermediate checkpoints when tuned for a single fixed terminal budget. The gap illustrates why cosine decay is not anytime: tuning for a long horizon can be far from optimal at shorter budgets. An analogous plot for the 300M model appears in Figure~\ref{fig:cos-transfer-300m} (Appendix~\ref{app:additional_figures}).}
  \label{fig:cosine_envelope}
\end{wrapfigure}
WSD is not a strictly anytime schedule, since it requires additional overhead in the form of storing checkpoints of the training run at sparse intervals, from which one can either extend the constant learning rate phase in case of new data, or decay in order to get the final performance. In the case of continued training after decaying, the FLOPs spent on the decay part of the scheduler need to be charged again, introducing additional cost. From a theory point of view, polynomial learning rate schedules are common in the stochastic optimization literature~\citep{shamir2013stochastic,
dieuleveut2016nonparametric,wu2022last}, and carefully designed horizon-aware
decay sequences exist for least squares~\citep{ge2019step,jain2019making}. However, unless weight averaging is used, these schedules generally do not
achieve optimal rates for SGD in linear
regression~\citep{ge2019step, wu2022last}.
In this work we investigate theoretically in infinite dimensional Gaussian linear regression, and empirically by training LLMs for up to 32$\times$ Chinchilla tokens, how $1/t^\gamma$ with iterate averaging,
constant learning rate with iterate averaging and WSD compare to cosine decay in long training runs,
and which of these schedules provides a viable anytime
alternative to cosine annealing. Concretely, we require two properties from an anytime scheduler: (i) the scheduler should not depend on token horizon, and (ii) for any intermediate duration $T$, it should be
competitive with a well-tuned cosine schedule run for $T$ steps, i.e. it
should track the \emph{cosine envelope} that these tuned cosine schedules
define across checkpoints in a long run. Prior work has compared schedules
across multiple training durations: \citet{hagele2024scaling} tune cosine
baselines separately per duration and compare cooldown runs against them, and
\citet{hu2024minicpm} report similar multi-horizon comparisons. In our evaluation, we ask whether a \emph{single fixed horizon-free
trajectory} can track the envelope of per-horizon-tuned cosine schedules at
almost every intermediate checkpoint of a run up to $32\times$ the
Chinchilla-optimal budget.

\begin{figure*}[!tp]
    \centering
    \includegraphics[width=1.0\linewidth]{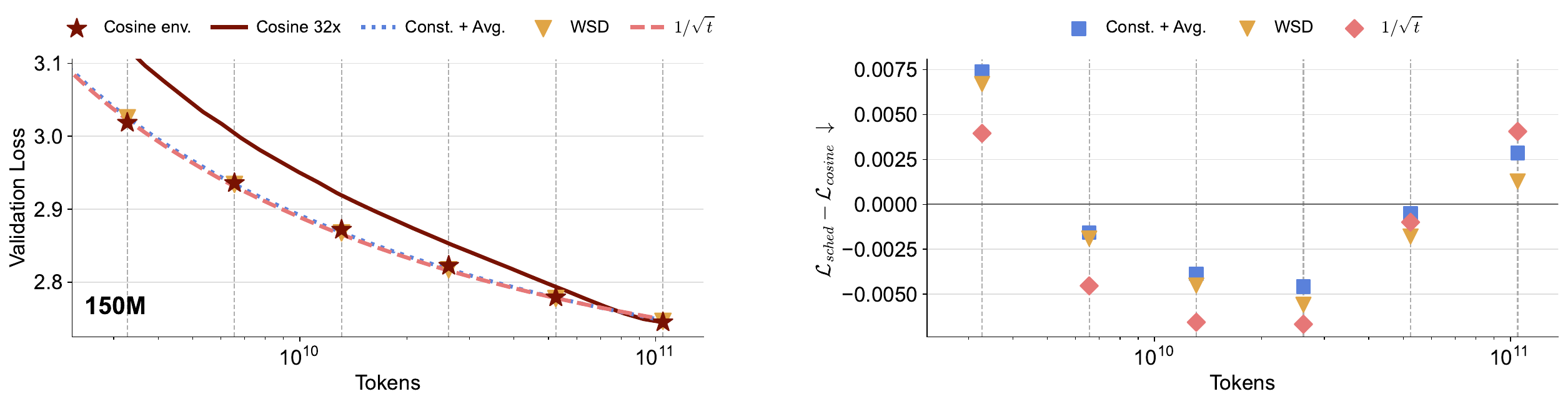}
    \vspace{0.5em}
    \includegraphics[width=1.0\linewidth]{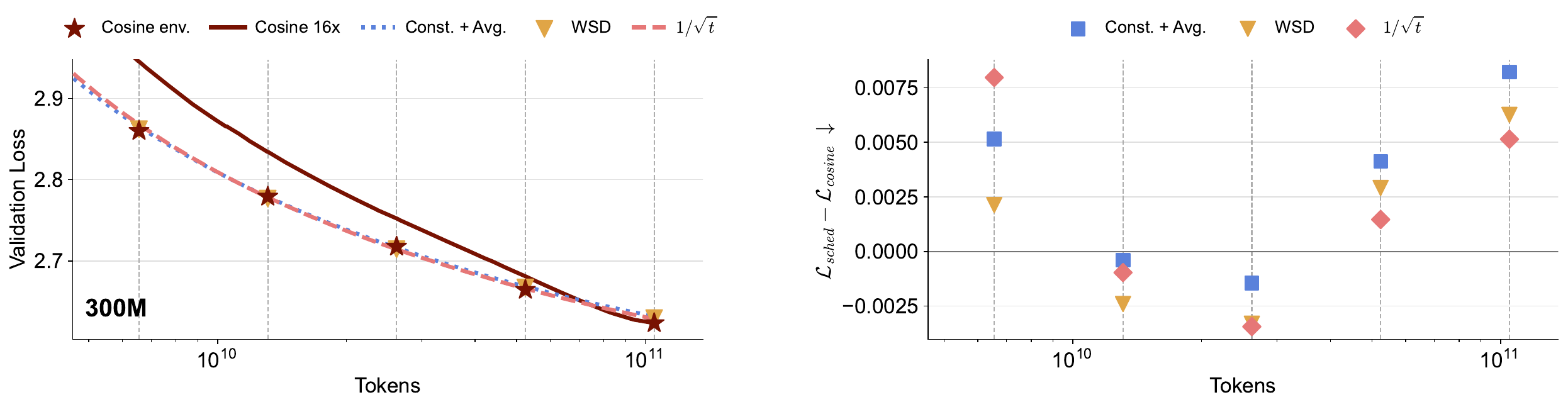}
  \caption{Anytime schedules vs.\ cosine decay; 150M (top) and 300M (bottom). \textbf{Left:} Validation loss vs.\ training compute at $1\times$--$32\times$ Chinchilla ($1\times$--$16\times$ for 300M). Cosine baselines are \emph{separate} runs, each tuned for its own duration; red stars mark the resulting per-duration optimal cosine envelope, and the red curve is the single cosine run tuned for the full horizon ($32\times$ and $16\times$, respectively). Constant LR with averaging and the $1/\sqrt{t}$ schedule are each a \emph{single} run trained to the full horizon and evaluated at intermediate checkpoints; the $1/\sqrt{t}$ schedule uses the multiplicative factor $\sqrt{\alpha/(t+\alpha)}$ with tuned $\alpha$. WSD branches from checkpoints of the constant-LR run at $90\%$ of each target budget and applies a linear decay over the final $10\%$ (one decay branch per budget, not a single trajectory). \textbf{Right:} Loss difference relative to cosine at each compute multiple (negative = better than cosine). Hyperparameters of the single run are chosen to closely track the cosine envelope and are reported in Table~\ref{tab:figure2-selected-hps}, with prospective transfer results shown in Figure~\ref{fig:hyperparameter-transfer-heatmaps}.}
  \label{fig:main_figure_both}
\end{figure*}

To address this question, we turn to the simplest theoretical proxy for pretraining, the quadratic model. Concretely, the bias-variance decomposition of the risk in quadratics~\citep{jain2018parallelizing,zou2023benign} suggests that weight averaging combined with polynomial learning rate decay can be horizon-free under certain spectral assumptions, while still achieving optimal rates. The same analysis predicts that when the variance contribution is small, e.g. at large batch size, maintaining a larger learning rate decreases the bias term faster, and decay becomes unnecessary. A central goal of this paper is to test whether these prescriptions survive contact with LLM pretraining at scale.

\paragraph{Contributions.} In Section~\ref{sec:theoretical_analysis}, we begin by deriving an anytime schedule in linear regression on quadratics, and specializing this theorem to power law covariance.~\citet{zhang2024does} (Corollary 2) have shown that for linear regression on power law data with capacity exponent $a$ and source exponent $b$, a constant learning rate with tail averaging is anytime for $b \leq a$ (and in particular it is optimal for $b=a$). In this work, we extend this regime by showing that a polynomial learning rate scheduler of the form $1/t^\gamma$ for $\gamma$ depending on $a$ and $b$ is anytime for $a < b \leq 2a$, and achieves the minimax optimal exponent for power law regression~\citep{zhang2024optimality}, up to logarithmic factors and constants.~\citet{zhang2024does} also provide a tight risk analysis in this regime for a constant learning rate with averaging. However, note that the constant rate is horizon dependent, as it scales as $N^{a/b - 1}$. Our results are compared to existing literature in Table~\ref{tab:related_rates}.

Following our theoretical analysis, we stress test the prescriptions of the quadratic model in Section~\ref{sec:empirical_findings}. We establish a baseline \textit{cosine envelope} by training 150M up to 32$\times$ Chinchilla and 300M LLMs up to 16$\times$ Chinchilla tokens using cosine decay tuned for each power-of-2 horizon in this interval. We then test several candidates for anytime schedules: a constant learning rate with iterate EMA and a polynomially decaying $1/\sqrt{t}$ with iterate EMA, as well as the nearly-anytime scheduler WSD, and seeing if they can closely track the cosine envelope. Figure~\ref{fig:cosine_envelope} shows why this is a hard baseline: a cosine tuned for a long horizon is substantially suboptimal at intermediate regimes. For the anytime candidates, we provide 2 evaluations: in Figure~\ref{fig:main_figure_both} we show that amongst all sweep configurations, there exists a \textit{single run} that can closely track the cosine envelope. Next, in Figure~\ref{fig:hyperparameter-transfer-heatmaps} we show that one can choose these runs prescriptively: following a burn-in, the optimal configuration at 4$\times$ Chinchilla across both models closely tracks the cosine envelope until the last evaluated point, while only incurring a gap at the 1$\times$ horizon. 

\section{Related Work}
\label{sec:related_work}
\paragraph{Finite dimensional SGD analysis.} There is a wide body of literature studying risk bounds in stochastic gradient descent, both in the finite dimensional regime and in the infinite dimensional/nonparametric setup.~\citet{polyak1992acceleration} introduced averaged SGD as an algorithm to achieve improved SGD convergence rates, while~\citet{bach2013non} later analyzed the rate obtained by it in least squares.~\citet{defossez2015averaged} analyzed constant learning rate with averaged iterates for least-squares, providing rates for the bias and variance terms, with similar proofs being shown in~\citet{jain2017markov,dieuleveut2016nonparametric}. This analysis has been extended to minibatch gradient descent~\citep{jain2018parallelizing} and streaming algorithms~\citep{frostig2015competing}. When the horizon is known in advance,~\citet{jain2019making} have shown that a carefully designed step size sequence can be minimax optimal for last iterate SGD, building up on previous work by~\citet{shamir2013stochastic,harvey2019tight}. \citet{ge2019step} showed that the final iterate under polynomially decaying
stepsizes is provably off by a
condition-number factor in the strongly convex case, and no better than
$d\sigma^2/\sqrt{T}$ against the minimax $d\sigma^2/T$ in the smooth case, while a geometrically decaying step-decay schedule is
near-minimax-optimal up to logarithmic factors, with
\citet{pan2021eigencurve} refining the schedule further.
\paragraph{Nonparametric Least Squares.} Recent work by~\citet{zhang2024does} has shown that for power law spectra and under certain conditions on source and capacity exponents, averaging can be minimax optimal.~\citet{zhang2024optimality} have
established minimax optimal rates for SGD-type methods under power-law
source and capacity conditions, characterizing the regimes in which plain
SGD attains the minimax rate and where acceleration is required. Other learning rate schedules have been analyzed using a bias variance decomposition technique~\citep{zou2023benign, wu2022last, wu2022power, meterez2025simplified}. From a statistical physics point of view,~\citet{bordelon2021learning} have established precise asymptotics for SGD in the overparameterized regime, with a similar analysis being done by~\citet{atanasov2024scaling, atanasov2025two} using tools from random matrix theory. 
\begin{table*}[!t]
    \centering
    \setlength{\tabcolsep}{5pt}
    \renewcommand{\arraystretch}{1.2}
    \resizebox{\textwidth}{!}{%
    \begin{tabular}{l l c l l l}
        \toprule
        Schedule & Averaging & Horizon-free & Rate & Regime & Reference \\
        \midrule
        Constant $\eta$ & yes & \checkmark & $N^{-(b-1)/a}$ & $1<b\leq a$ & \citet{zhang2024does} \\
        Constant $\eta \propto N^{a/b-1}$ & yes & $\times$ & $N^{-(1-1/b)}$ & $a<b \leq 2a$ & \citet{zhang2024does} \\
        Horizon-optimal schedule & last iterate & $\times$ & $N^{-(b-1)/\max\{a,b\}}$ & both & \citet{bordelon2026theory} \\
        \midrule
        $1/t^{\gamma^\star}$, $\gamma^\star = 1 - a/b$ & yes & \checkmark & $N^{-(1-1/b)}$ & $1 < a < b \leq 2a$ & \textbf{this work} (Thm~\ref{thm:tgamma_rate}, Cor.~\ref{cor:optimal_hps}) \\
        \bottomrule
    \end{tabular}}
    \caption{Excess-risk rates for SGD in infinite dimensional linear regression and closely related models. We use capacity exponent $a$ and source exponent $b$. Horizon-free means that the learning rate scheduler does not depend on the training horizon $N$. Rates are presented up to constants and logarithmic factors. }
    \label{tab:related_rates}
\end{table*}
\paragraph{Theory and Practice for Learning Rate Scheduling.}~\citet{defazio2023optimal} have analyzed the linear decay schedule and have proposed further refinements for this schedule using gradient norms. Linearly decaying to zero has been studied empirically by~\citet{bergsma2025straight}, achieving competitive performance to cosine annealing.~\citet{defazio2024road} have proposed a schedule-free optimizer, which takes advantage of tail averaging to remove the need for learning rate scheduling. Recent works leverage power-law kernel and random-feature regression models to derive optimal learning-rate schedules for last-iterate
performance~\citep{li2026optimal,bordelon2026theory}.
Closely related to our work,~\citet{hagele2024scaling} show that a constant learning rate with a decay at the end (a schedule similar to the trapezoidal scheduler proposed by~\citet{zhai2022scaling} and WSD~\citep{hu2024minicpm}) performs close to cosine decay at matched training lengths. Moreover, they find that stochastic weight averaging (SWA) with a fixed window improves models throughout training, but does not fully close the gap to learning rate decay. Our results complement this finding: guided by the quadratic analysis, we use an EMA window that grows proportionally with elapsed training time, and ask whether a single fixed configuration can track the envelope of per-horizon-tuned cosine baselines at every intermediate checkpoint of runs extending to $32\times$ Chinchilla. In this regime, we find that averaging can track the cosine envelope with no decay phase.~\citet{schaipp2025surprising} show that last-iterate bounds from convex optimization closely predict the behavior of practical schedules, explaining the benefit of cooldowns.~\citet{tissue2026scaling,luo2025multi,li2026functional} model the loss curve as a function the scheduler, predicting behaviour across horizons and schedulers. More recently, stochastic weight averaging has obtained impressive results in the AlgoPerf competition~\citep{kasimbeg2025accelerating, dahl2023benchmarking}, with~\citet{ajroldi2025and} providing a systematic study of when averaging helps. Weight averaging has also been commonly applied in image generation and diffusion~\citep{yazici2018unusual, karras2024analyzing,song2020score}. Other variants of weight averaging have been used in practice~\citep{kaddour2022stop,sanyal2023early,li2024switch,morales2024exponential,arpit2022ensemble, wortsman2022model}. 
It has been empirically observed that averaging can also lead to flatter minima and improved generalization performance on multiple image tasks~\citep{izmailov2018averaging}. 
\paragraph{Continual Learning and Anytime Training.} Generally, continual learning refers to training a model on a sequence of (possibly orthogonal) tasks, ensuring that the model does not forget any of them~\citep{wang2024comprehensive, aljundi2019online, veniat2020efficient, cossu2024continual, ramasesh2021effect, mehta2023empirical, joudaki2025barriers}. Anytime training~\citep{caccia2022anytime,ibrahim2024simple, gupta2023continual} refers to training a model optimally without having access to the total amount of steps i.e. the training horizon. While the two themes appear separate, we believe they are closely related, as anytime pretraining is a necessary stepping stone towards continual learning. A possible technique for anytime pretraining involves re-warming and re-decaying the learning rate after each stage of training~\citep{ibrahim2024simple}, which has been explored by~\citet{wen2024understanding} in the context of Warmup-Stable-Decay (WSD)~\citep{hu2024minicpm} schedules. Learning dynamics of WSD have been studied empirically by~\citep{belloni2026universal}.~\citet{singh2025beyond} show that infinite schedules which avoid re-warming altogether consistently improve continual pretraining over repeated cosine cycles.

\section{Theoretical Analysis}
\label{sec:theoretical_analysis}
In this section, we provide an upper bound on the risk of SGD with iterate averaging, and show that the $1/t^\gamma$ schedule achieves the minimax rate exponent for power law Gaussian data established by~\citet{zhang2024optimality}, up to logarithmic and constant factors. To complement these asymptotic results, Figure~\ref{fig:noalpha_synthetic_0.1} evaluates the infinite dimensional (approximated by $d \gg N$) recursion for Gaussian linear regression with constant noise variance. The lower noise experiment in Figure~\ref{fig:noalpha_synthetic_0.01} shows that decreasing $\sigma^2$ delays the variance-dominated crossover while leaving the schedules horizon-free.
\begin{figure*}[!htp]
    \centering
    \includegraphics[width=1.0\textwidth]{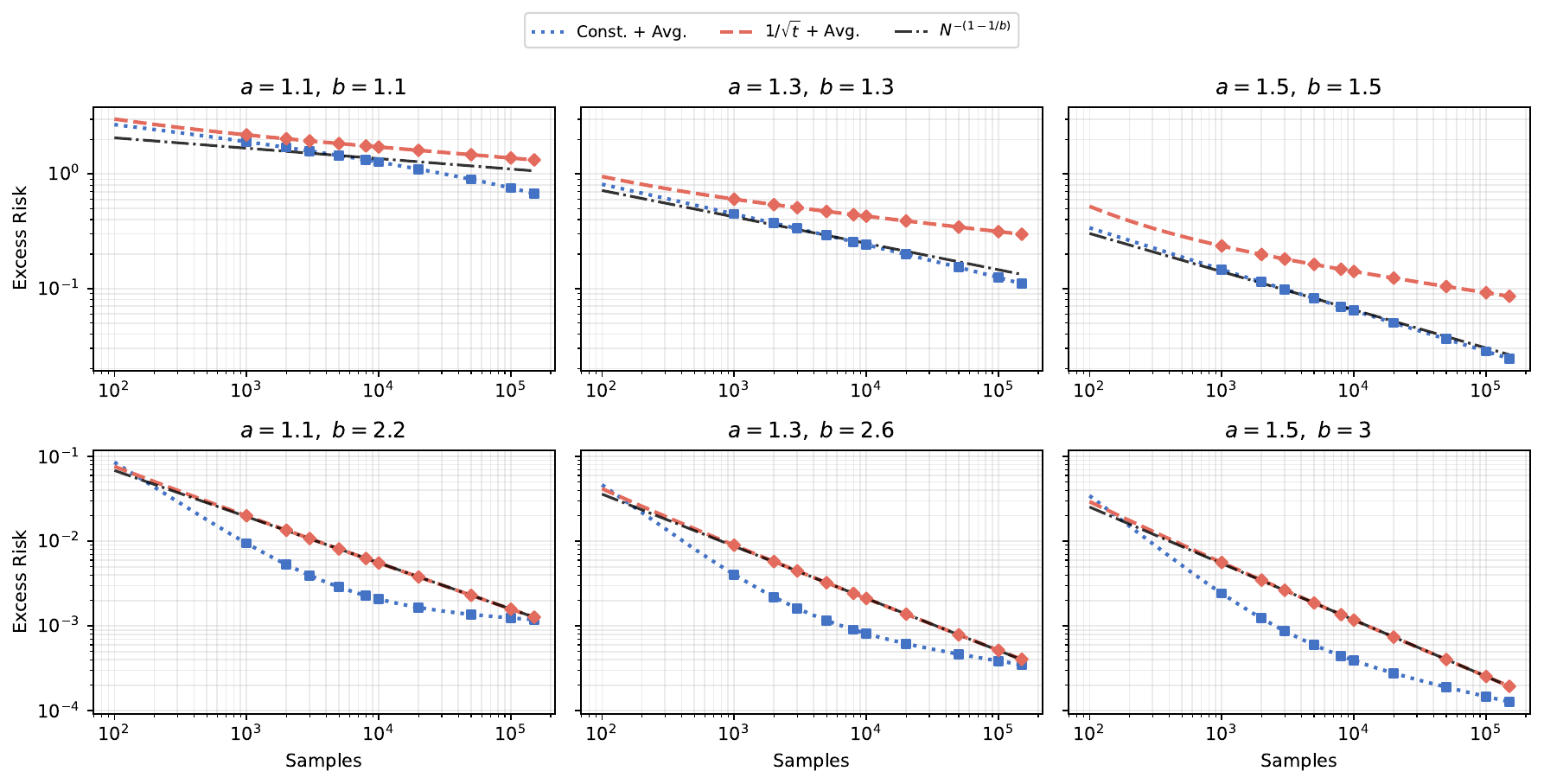}
    \caption{Synthetic excess risk for SGD on Gaussian linear regression with diagonal Hessian $\lambda_i=i^{-a}$, target coordinates $(w_i^\star)^2=i^{-(b-a)}$, dimension $d=500{,}000$, $N=150{,}000$, and $\sigma^2=0.1$. Constant and $1/\sqrt{t}$ learning rates both use iterate averaging. For each panel and scheduler, we sweep fixed multipliers $c\in\{10^{-4},2\!\times\!10^{-4},5\!\times\!10^{-4},7\!\times\!10^{-4},10^{-3},2\!\times\!10^{-3},5\!\times\!10^{-3},10^{-2},2\!\times\!10^{-2},3\!\times\!10^{-2},5\!\times\!10^{-2},7.5\!\times\!10^{-2},0.1,0.2,0.3,0.5,0.8,1,2,3,5,10\}$, where the actual step size is $c/\tr(\H)$ for constant LR and $c/(\tr(\H)\sqrt{t})$ for inverse square root. We select the learning rate using the best loss at $N=5{,}000$ and evaluate that same trajectory at every later checkpoint. The dashed trace shows the predicted rate $N^{-(1-1/b)}$. We show the synthetic results for $\sigma^2=0.01$ in Figure~\ref{fig:noalpha_synthetic_0.01}.}
    \label{fig:noalpha_synthetic_0.1}
\end{figure*}
\paragraph{Setup and notation.} Throughout the theory section of the manuscript, we study linear regression on Gaussian data with batch size $1$ and independent additive noise in the samples. We use the notation $f \lesssim g$ to say there exists some positive constant $c$ such that for any $x$ in the domain we have $f(x) \leq c g(x)$. We also denote $f \eqsim g$ if $f(x) \lesssim g(x) \lesssim f(x)$ for all $x$. Note that we will also absorb $\log$ factors in the $\lesssim$ notation, stating where we do so. We also denote the induced norm $\|\w\|_{\mathbf{A}}^2 = \w^\top \mathbf{A}\w$. For a diagonal matrix we define its restriction to entries between $i$ and $j$ where $1 \leq i \leq j \leq \infty$ as $\LLambda_{i:j} = \diag(\lambda_i, \dots, \lambda_j)$. We denote the i.i.d. input-output stream as $\{(\x_t, y_t)\}_{t\geq 1}$ where each $\x_t \in \R^d$ and $y_t \in \R$ satisfy:
\begin{align*}
    \x_t \sim \N(0, \H) && y_t|\x_t \sim \N(\x_t^\top \w^\star, \sigma^2)
\end{align*}
where $\w^\star$ is the minimizer and $\sigma^2$ is the variance of the additive noise. We take the convention that $\lambda_{\max} = \lambda_1$ and eigenvalues are sorted in nonincreasing order. We define the risk under the mean squared error loss to be:
\begin{align*}
    \risk(\w) = \frac{1}{2} \E_{\x, y}(\x^\top \w - y)^2
\end{align*}
We denote the learning rate at time $t\geq 1$ by $\eta_t$ and in general $\eta_t \coloneq \eta \cdot f(t, N)$ for some function $f$ of the current time $t$ and optionally the end time $N$, and for a constant base learning rate $\eta \in (0, 1)$ that does not depend on $N$. We take $\w_0=0$ as the initialization without loss of generality, and count updates from $t=1$, so update $t$ uses $(\x_t,y_t)$ and produces $\w_t$. We also denote $\bw_{s:s+T} \coloneq \frac{1}{T} \sum_{i=s}^{s+T-1} \w_i$ to be the average of the last $T$ iterates starting from step $s\geq1$. The main proof technique used to establish the rates is based on the bias-variance decomposition of the risk~\citep{jain2018parallelizing,jain2017markov,zou2023benign,wu2022last,wu2022power,meterez2025simplified}. All the proofs are deferred to Appendix~\ref{app:proofs}.
\subsection{Main results}
We now state the main result regarding the excess risk rate of SGD with $\eta_t = 1/t^\gamma$. We state Theorem~\ref{thm:tgamma_rate} using $s=\Theta(N)$ in order to simplify notation, but we provide the general bound in the proof.
\begin{theorem}
    \label{thm:tgamma_rate}
    Let $\LLambda = \diag(\lambda_i)_{i\geq 1}$ be the eigenvalues of $\H$ and learning rate $\eta_t = \eta/t^\gamma$ for $t \geq 1$, $\gamma \in (0, 1/2]$, and some constant $0 < \eta \lesssim 1/\tr(\H)$ independent of $N$, and $\|\w^\star\|_\H^2 \lesssim \sigma^2$. We have for $s = \Theta(N)$:
    \begin{align*}
        \E \risk(\bw_{s:s+N}) - \frac{\sigma^2}{2} &\lesssim \frac{1}{N} \|\w^\star\|_{\LLambda_{1:k^\star}}^2 + \|\w^\star\|_{\LLambda_{k^\star:\infty}}^2 + \frac{k^\star \sigma^2}{N}
     + \sigma^2
    \sum_{k>k^\star}
    \left(
        \eta^2 \lambda_k^2\, N^{1-2\gamma}
    + \eta \lambda_k\, N^{-\gamma}
    \right)
    \end{align*}
    for $k^\star
\coloneqq
\max\Big\{k:\ \lambda_k \ge \tfrac{\log N}{\eta N^{1-\gamma}}\Big\}$, where $\lesssim$ absorbs absolute constants and the $\log N$ term for $\gamma = 1/2$.
\end{theorem}
Following~\citet{jain2017markov}, we have one contribution to the rate coming from the mean SGD process fitting the data and approaching the minimizer $\w^\star$ termed \textit{bias}, and another term coming from the additive and sampling noise, known as \textit{variance}. In finite dimension, taking $k^\star = d$ recovers the asymptotic $d \sigma^2/N$ variance rate~\citep{polyak1992acceleration,bach2013non,ge2019step}. 

In Corollary~\ref{cor:optimal_hps}, we specialize the covariance spectrum to a power law setting, using source and capacity exponents~\citep{caponnetto2007optimal,lin2024scaling,bordelon2024dynamical,paquette20244+}, and compute the optimal choice of gamma as a function of these exponents.

\begin{corollary}
\label{cor:optimal_hps}
Suppose the following capacity and source conditions
\[
    \lambda_k \eqsim k^{-a},
    \qquad
    \lambda_k (w_k^\star)^2 \eqsim k^{-b},
    \qquad
    1 < a<b\leq 2a
\]
and assume $\|\w^\star\|_\H^2 \lesssim \sigma^2$. Then the rate-optimal decay exponent is
\[
    \gamma^\star = 1-\frac{a}{b}\in\left(0,\frac12\right].
\]
In particular, for $\eta_t=\eta/t^{\gamma^\star}$ for $t \geq 1$ and
$s=\Theta(N)$, Theorem~\ref{thm:tgamma_rate} gives
\[
    \E\risk(\bw_{s:s+N})-\frac{\sigma^2}{2}
    \lesssim
    N^{-\left(1-\frac1b\right)},
\]
where $\lesssim$ hides constants and the $\log N$ factors.
\end{corollary}
From Corollary~\ref{cor:optimal_hps} we see that the optimal $\gamma^\star$ depends on the spectral properties of the data. Within the regime $1<a<b\leq 2a$, larger values of $b$ correspond to signal that decays more rapidly across the eigendirections of $\H$, and therefore to faster learning-rate decay: $\gamma^\star=1-a/b$ increases from $0$ to $1/2$. Conversely, as $b$ approaches $a$, $\gamma^\star$ approaches $0$, for which prior work shows that constant-stepsize SGD with tail averaging can attain the corresponding optimal rates~\citep{zhang2024does,zou2023benign}. Intuitively, when $b < a$, more target energy lies in tail directions, so maintaining a larger learning rate improves contraction of the bias. 
\paragraph{Constant learning rate.}\citet{jain2018parallelizing} (Theorem 1) have shown that in the classical regime, a constant learning rate with tail averaging suffices to achieve the minimax rate in the variance of $\sigma^2 d/N$. In the high dimensional regime, ~\citet{zou2023benign,zhang2024does} have shown that a constant learning rate with tail averaging can achieve the minimax SGD rate under source and capacity constraints. For $b \leq a$, ~\citet{zhang2024does} (Corollary 2) show that a constant learning rate with averaging is a true anytime schedule, whereas for $b > a$ the learning rate has to scale with the horizon $N$. We have tried a constant learning rate with weight averaging in Figure~\ref{fig:main_figure_both}, showing that this schedule performs competitively with cosine and $1/t^\gamma$.
\paragraph{WSD.} WSD learning rate schedules have shown great promise empirically, being adopted by frontier labs~\citep{team2025kimi}. Recall that the WSD schedule, as introduced by~\citet{hu2024minicpm}, consists of a warmup stage, a constant learning rate stage and a decay over the last $10\%$ of the training run. Related work by~\citet{bordelon2026theory} (Result 2) and~\citet{li2026functional} derive optimal last-iterate horizon dependent learning rate schedules in a power law random feature model, which recovers our linear regression setup for identity features. Informally, the WSD schedule is characterized by how long the constant learning rate region should lsat for as a function of the end time. \citet{bordelon2026theory} (Result 2) recover 2 separate schedules: for $b > a$, it is optimal to decay for the whole training run, whereas for $b < a$ they recover a schedule resembling WSD where the constant learning rate fraction scales as a function of the end time and the source and capacity exponents. The main focus of our work is complementary: we study whether horizon-free learning-rate trajectories, combined with iterate averaging, can achieve optimal rates without knowing the stopping time.

\section{Empirical Findings}
\label{sec:empirical_findings}

\begin{table*}[!t]
    \centering
    \small
    \setlength{\tabcolsep}{8pt}
    \begin{tabular}{llcccc}
        \toprule
        Model & Schedule & Evaluated through & $\eta$ & $\beta_2$ & $\alpha$ \\
        \midrule
        150M & Constant & $32\times$ & 0.003 & 0.99 & -- \\
        150M & $1/\sqrt{t}$ & $32\times$ & 0.003 & 0.98 & 51,200 \\
        \midrule
        300M & Constant & $16\times$ & 0.003 & 0.95 & -- \\
        300M & $1/\sqrt{t}$ & $16\times$ & 0.003 & 0.99 & 51,200 \\
        \bottomrule
    \end{tabular}
    \caption{Selected hyperparameters for the fixed constant-LR and $1/\sqrt{t}$ trajectories in Figure~\ref{fig:main_figure_both}. Each row is one trajectory evaluated at every displayed checkpoint through the stated horizon. All runs use a 5,000-step warmup, $\beta_1=0.9$, $\epsilon=10^{-8}$, zero weight decay, and batch size 256 for 150M or 512 for 300M.}
    \label{tab:figure2-selected-hps}
\end{table*}
In this section, we describe how we compare our learning-rate schedules to cosine decay. We focus on anytime schedulers, which are intended to operate without prior knowledge of the training horizon. For each anytime schedule we report results from \textit{a single training run} using a hyperparameter setting chosen to track cosine performance across the entire training trajectory, rather than maximizing final-step performance. As a point of comparison, Figure~\ref{fig:hyperparameter-transfer-heatmaps} shows the transfer across subsequent horizons when optimally tuning for an earlier one, and Appendix~\ref{app:additional_figures} includes additional loss curves obtained by selecting the best hyperparameter setting separately for each training budget.
\paragraph{Architecture and Dataset.}
We pretrain 150M and 300M models based on the OLMo codebase~\citep{groeneveld2024olmo,olmo20242,olmo2025olmo}. The architectural details are reported as depth, number of heads and width: 150M (12, 16, 1024), 300M (24, 16, 1024). The models are trained using AdamW with no weight decay, fixing $\epsilon = 10^{-8}$ and momentum $\beta_1=0.9$. The 150M cosine runs sweep over $\eta\in\{0.0003,0.001,0.003,0.01\}$ at the $1\times$--$8\times$ horizons and $\eta\in\{0.0003,0.001,0.003\}$ at the $16\times$--$32\times$ horizons, with $\beta_2\in\{0.95,0.98,0.99\}$. 

Note that our theoretical guarantees determine the asymptotic decay exponent only up to logarithmic and constant factors, however in practice a finite time offset is necessary. We parameterize the schedule as
\[
\eta_t=\eta\sqrt{\frac{\alpha}{t+\alpha}},
\]
where \(\alpha>0\) is a tuned offset controlling the transition from an approximately constant learning rate at \(t\ll\alpha\) to \(1/\sqrt{t}\) decay at \(t\gg\alpha\). Once \(\eta\) and \(\alpha\) are fixed, the schedule depends only on the current step \(t\), rather than on the terminal training horizon. In Figure~\ref{fig:main_figure_both}, \(\alpha\) is selected as part of a single configuration that is then evaluated at every intermediate horizon; it is not retuned separately at each endpoint. Figure~\ref{fig:inverse-sqrt-alpha-sensitivity-4x} examines sensitivity to this choice by fixing the remaining hyperparameters to those selected at \(4\times\) Chinchilla and comparing all executed values of \(\alpha\) across the evaluated horizons. We refer to the resulting shifted schedule as \(1/\sqrt{t}\) throughout for the remainder of this section.

For 150M, the constant learning-rate runs sweep over $\eta \in \{0.0003,0.001,0.003\}$ and $\beta_2\in\{0.95,0.98,0.99\}$, while the inverse-square-root runs sweep over the same learning rates, $\beta_2\in\{0.98,0.99\}$, and offsets $\alpha\in\{12{,}800,25{,}600,51{,}200\}$. For 300M, all schedulers sweep over $\eta\in\{0.0003,0.001,0.003\}$ and $\beta_2\in\{0.95,0.99\}$, with the inverse-square-root runs additionally sweeping over offsets $\alpha\in\{12{,}800,25{,}600,51{,}200\}$. All the WSD runs are continued from the checkpoints saved during the constant learning rate sweep, and we plot the best (without averaging) at each point. We train without z-loss and with a sequence length $L=1024$. We train our models on the C4 dataset~\citep{2020t5}, and we use the T5 tokenizer, and we do not repeat over the data, with all our runs being fully online.

\paragraph{Training and evaluation details.}
Following the Chinchilla~\citep{hoffmann2022training} calculation of $20$ tokens-per-parameter (TPP), we take as $1\times$ Chinchilla for the 150M model to be 3.3B tokens and 6.6B tokens for the 300M model. Unless specified otherwise, we train all our models close to the critical batch size for $1 \times$ Chinchilla, based on~\citet{zhang2024does}, which is 256 for 150M and following the $\sqrt{N}$ (for $N$ denoting the total data size scaling), we approximate at $512$ the critical batch size for the 300M model. We fix the warm-up duration to be $5000$ steps (corresponding to approximately $40\%$ of the respective $1\times$ Chinchilla number of tokens for each of the models), for all training runs (the only exception being the $B=4096$ runs, where we do $500$ warmup steps). For the 150M model, we train a cosine baseline at each power-of-2 multiple of Chinchilla from 1$\times$ to 32$\times$. For the $1/\sqrt{t}$ and constant with averaging runs, we train directly for the $32\times$ and compare with cosine at the intermediate points. Moreover, we save checkpoints at $90\%$ of each Chinchilla multiple for the constant learning rate run, and do a linear decay from each of these points in order to implement the WSD schedule. We follow a similar procedure for the 300M models, but stop at $16\times$. We keep the batch size fixed to the $1\times$ Chinchilla critical batch size in order to mimic an anytime training regime in which one would start training a model for $1 \times$ Chinchilla, and more data would be streamed in the training process down the line. We leave as future work how to combine anytime schedules with a batch ramp-up scheme.

The selected hyperparameters for the fixed constant-LR and $1/\sqrt{t}$ trajectories in Figure~\ref{fig:main_figure_both} are reported in Table~\ref{tab:figure2-selected-hps}. These configurations were chosen in order to minimize the average gap to the cosine envelope across all points, showing the existence of a \textit{single run} for the 2 schedulers that closely tracks the cosine envelope at all points. For a prospective comparison, we show in Figure~\ref{fig:hyperparameter-transfer-heatmaps} that one can choose the $4 \times$ optimal hyperparameters and still stay close to the cosine envelope, while incurring a penalty at the $1 \times$ multiplier.
\paragraph{Averaging.} We maintain EMA parameters 
\[
\bw_{t+1} = \tau_t \bw_{t} + (1-\tau_t) \theta_t \qquad \tau_t = 2^{-f/t}
\] 
At time $t$, the half life $h$ satisfies 
\[
(\tau_t)^h = \frac{1}{2} \implies h = \frac{t}{f}
\]
For all runs, we maintain multiple EMAs for $f \in \{6.25, 12.5, 25.0, 50.0, 100.0\}$ and treat $f=0.0$ separately by fixing it to mean the last iterate. Thus, for $f=25$ for example, the half life of the EMA is roughly $t/25 \approx 4\% \text{ of } t$. All coefficients depend only on the current time and a fixed value of \(f\), not on the terminal training horizon. The averaging procedure is therefore online and horizon-free. Our implementation maintains 5 model copies (corresponding to each value of $f$) in addition to the trainable model. Unless mentioned otherwise, the constant learning rate and $1/\sqrt{t}$ curves in the main figures use \(f=25\) throughout, whereas cosine and WSD use the last iterate, \(f=0\). We do not select \(f\) separately at each training horizon. 
We decided to fix $f=25$ across all experiments that use an EMA as this always corresponded to the best loss out of the EMA factors in our experiments.
\paragraph{Choice of $\gamma$.} We experimentally compare $\gamma = 0$ (constant learning rate) and $\gamma=1/2$, meaning $1/\sqrt{t}$, for which we provide the following explanation.~\citet{mlodozeniec2025completed, bjorck2024scaling} have shown that the optimal learning rate in practice approximately scales proportionally to $1/\sqrt{N}$. In the quadratic regime, the bias along eigendirection $i$ contracts at a rate controlled by the cumulative step size, roughly as $\exp\big(-\lambda_i \sum_{s\le t}\eta_s\big)$. Consequently, the mean process continues to make substantial progress up to time $N$ in all directions satisfying
\begin{align*}
    \lambda_i \sum_{s=1}^N \eta_s \gtrsim 1,
    \qquad\ \implies \qquad
    \lambda_i \gtrsim \frac{1}{\sum_{s=1}^N \eta_s}.
\end{align*}
For $\eta_s \eqsim 1/\sqrt{N}$, we have $\sum_{s=1}^N \eta_s \eqsim \sqrt{N}$. In order to ensure that an anytime scheme has the same scaling at any time $t$, one reasonable option is to choose $\eta_s = 1/\sqrt{s}$, since $\sum_{s=1}^t s^{-1/2} \eqsim \sqrt{t}$, thus motivating the choice of $\gamma=1/2$. We believe that other schedules around this value might also work~\citep{bjorck2024scaling,shuai2024scaling}, but for all our experiments, we chose to keep $1/2$. We formalize this bias--variance decomposition and the resulting recursion in Appendix~\ref{app:proofs},~\eqref{eq:bias_and_variance_eta}.
\subsection{Anytime Performance and Early-to-Late Hyperparameter Transfer}
\begin{figure*}[!htp]
    \centering
    \includegraphics[width=1.0\linewidth]{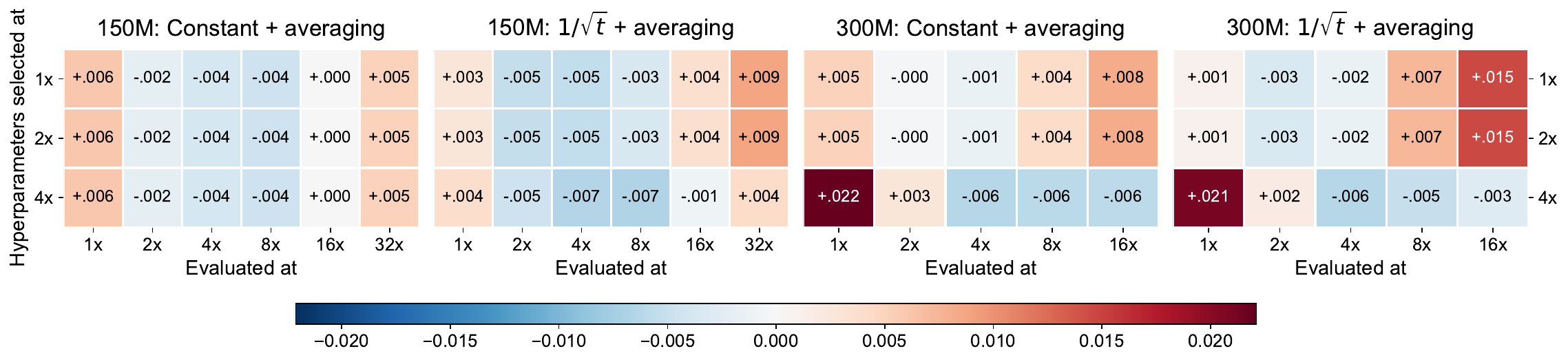}
    \caption{Prescriptive transfer for constant learning rate and $1/\sqrt{t}$. The rows select the configuration with the lowest validation loss at $1\times$, $2\times$, or $4\times$ Chinchilla and show the performance of that configuration at later points. Values are the $4\%$-EMA validation loss relative to the cosine envelope, and negative values indicate improvement over cosine. The hyperparameters are shown in Table~\ref{tab:optimal-hps-150m} and Table~\ref{tab:optimal-hps-300m}.}
    \label{fig:hyperparameter-transfer-heatmaps}
\end{figure*}
\paragraph{Fixed run comparison.}
Figure~\ref{fig:main_figure_both} shows that constant learning with averaging rate and \(1/\sqrt{t}\) with averaging, provide competitive horizon-free alternatives to cosine decay. For each scheduler, we fix a single hyperparameter configuration and evaluate the same trajectory at every intermediate horizon. We also compare against WSD, based on~\citet{team2025kimi,hu2024minicpm}: a warmup phase, followed by a constant learning rate until \(90\%\) of each target horizon and a linear decay over the remaining \(10\%\), ending at \(10\%\) of the peak learning rate. In practice, we save checkpoints at \(90\%\) of each intermediate horizon and initiate a separate decay from each checkpoint. Unlike the constant and \(1/\sqrt{t}\) trajectories, WSD is not horizon-free because the start of its decay depends on the selected stopping horizon.
\paragraph{Prospective hyperparameter transfer.}
The fixed trajectories in Figure~\ref{fig:main_figure_both} establish that a single configuration can remain competitive across the evaluated horizons, but practical use also requires a way to select that configuration. At a checkpoint $T_0$, we select the configuration with the lowest validation loss among the executed sweep. We then continue with the same configuration and evaluate it at every subsequent horizon, without using later-horizon losses for selection. We measure transfer relative to the independently tuned cosine envelope in Figure~\ref{fig:hyperparameter-transfer-heatmaps}, with negative values indicating improvement over cosine. Across the evaluated settings, $4\times$ Chinchilla is the earliest common calibration horizon that transfers consistently for both schedules and model sizes, with 1$\times$ also being close to it. For the 150M model, configurations selected at $4\times$ remain close to the cosine envelope through $32\times$: the maximum gaps are $0.0051$ for constant learning rate and $0.0041$ for $1/\sqrt{t}$. For the 300M model, both configurations selected at $4\times$ outperform the cosine envelope at every evaluated horizon from $4\times$ through $16\times$. This yields a practical anytime protocol: tune at an intermediate calibration checkpoint and continue training. 

\subsection{Large Batch Setting}

\begin{figure*}[!tp]
    \centering
\includegraphics[width=1.0\linewidth]{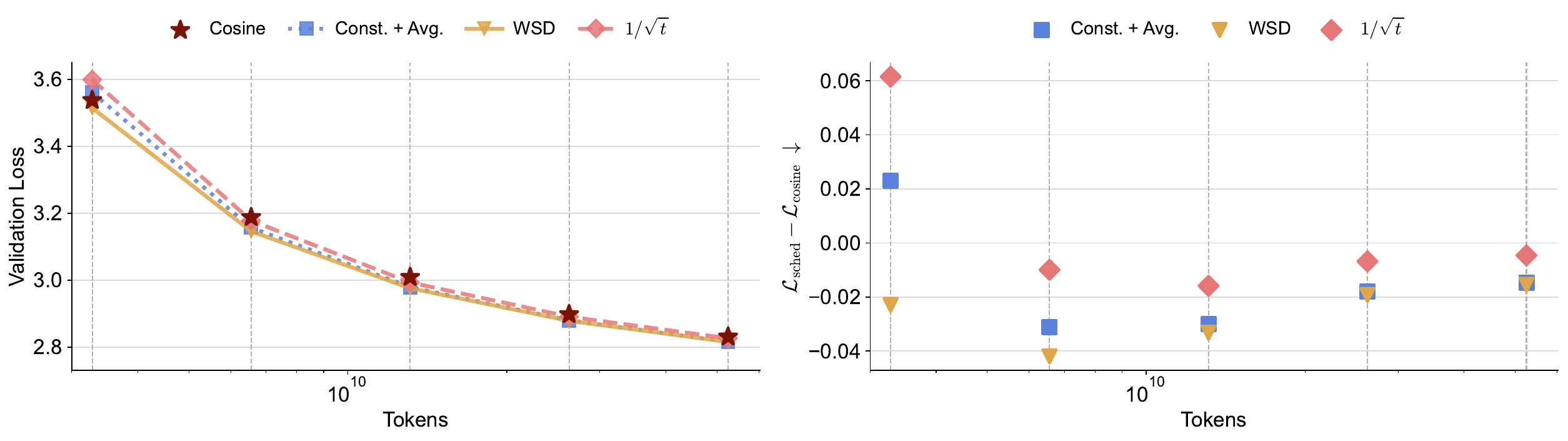}
    \caption{For a 150M-parameter model trained with batch size $4096$ and a 500-step warmup, we compare cosine decay to constant learning rate with averaging, a $\sqrt{\alpha/(t+\alpha)}$ schedule with averaging, and WSD across end times ranging from $1\times$ to $16\times$ Chinchilla compute. \textbf{Left:} Validation loss versus training compute for each schedule. \textbf{Right:} Loss improvement over cosine at each compute multiple, where a negative value is an improvement over cosine. A per-duration optimal hyperparameter version of this plot is in Figure~\ref{fig:large_batch_optimal_hps} (Appendix~\ref{app:additional_figures}).}
\label{fig:large_batch}
\end{figure*}
Based on the quadratic analysis, learning rate decay has the role of reducing variance. However, there is an inherent trade-off: a smaller learning rate has a smaller variance, but also a worse contraction in the bias term. In this section we test whether this theoretical prediction from the quadratic model is also relevant in LLM pretraining: if the variance is already very small, is decay still necessary to get a good final performance? Concretely, in the quadratic regime, SGD can be written as GD plus a batch-noise term:
\[
\theta_{t+1}-\theta^\star \approx (I-\eta H)(\theta_t-\theta^\star) + \eta \xi_t,
\]
where $\E[\xi_t] = 0$ and $\cov(\xi_t) \propto 1/B$. Thus, moving to very large batches suppresses gradient noise, making learning rate decay less necessary for controlling the variance term, and a larger learning rate more advatangeous for faster bias decay. In this section, we repeat the experiment from Figure~\ref{fig:main_figure_both} in a very large-batch regime (batch size $4096$), well beyond the critical batch size (CBS). While this regime is not relevant for practice, the large batch size reduces the variance in the risk linearly, stress testing the predictivity of the quadratic model for LLM pretraining experiments. Figure~\ref{fig:large_batch} is consistent with this prediction, since for $B=4096$, a constant learning rate with averaging is generally better than cosine for all horizons beyond $1\times$ Chinchilla. The $1/\sqrt{t}$ schedule also remains competitive and improves over cosine, but underperforms constant, which is again consistent with unnecessary decay in this regime. Lastly, the difference between constant with averaging and WSD decreases with $N$. We believe this is due to a higher order effect caused by the deterministic edge of stability as studied in \citet{damian2023selfstabilization}. However, we leave the precise study of this phenomenon to future work.

\section{Discussion and Limitations}
\subsection{Limitations} 
\paragraph{Theoretical Limitations.} From a theoretical point of view, our upper bound matches the minimax lower bound exponent established by~\citet{zhang2024optimality} up to constants and logarithmic factors. An important future work would be to remove the $\log N$ factor while still having an anytime scheduler. 
\paragraph{Hyperparameter selection.} Once the hyperparameters such as learning rate are fixed, both the constant and \(1/\sqrt{t}\) schedules are horizon-free: their values at step \(t\) do not depend on the eventual stopping time. For Figure~\ref{fig:main_figure_both}, we select one configuration per scheduler from the completed sweep based on its aggregate performance across the evaluated checkpoints and use that same configuration throughout the entire trajectory, without retuning at individual endpoints. This demonstrates the existence of a single trajectory that remains competitive with independently tuned cosine runs across a broad range of horizons. We additionally evaluate the practical anytime setting by selecting the optimal hyperparameters at \(4\times\) Chinchilla and continuing the same runs unchanged through \(32\times\) for the 150M model and \(16\times\) for the 300M model. The selected configurations remain competitive at subsequent horizons, showing that hyperparameters chosen early in training can transfer to substantially longer runs without advance knowledge of their final duration. However, the $4 \times$ is a heuristic choice, and it is not clear how to choose the point for which to tune the hyperparameters for in order to obtain transfer down the line. We believe one approach to this would be to build scaling laws across the schedulers in order to understand the optimal tuning point, and we leave this important extension to future work.
\paragraph{Scope.} Our empirical conclusions are limited to 150M- and 300M-parameter models trained with AdamW, no weight decay, and the batch-size regimes considered here. In particular, weight decay, alternative optimizers, and substantially larger model scales may change the relative behavior of the schedules.

\subsection{Conclusion} In this work, we study the feasibility of horizon-free learning-rate schedules for LLM pretraining, specifically a constant learning rate and $1/\sqrt{t}$ combined with EMA. Across well-tuned cosine-decay baselines for 150M and 300M parameter models, spanning $1\times$ to $32\times$ (and $1\times$ to $16\times$, respectively) Chinchilla regimes, we find that these anytime schedules are close to cosine (within small gaps) at intermediate checkpoints without requiring knowledge of the total training horizon. This suggests a promising direction for continual and open-ended training.
On the theory side, for linear regression with power-law spectra, we derive convergence rates as a function of source and capacity. We show that under spectral constraints, a schedule of the form $1/t^\gamma$ attains optimal rates up to logarithmic and constant factors, with $1/\sqrt{t}$ as a notable special case. We also compare against WSD empirically: while WSD is not strictly horizon-independent because it relies on additional checkpoints near $90\%$ of the run, it achieves performance comparable to cosine and the proposed anytime schedules. Overall, our results indicate that anytime schedules can be competitive with carefully tuned cosine baselines, emphasize the practical importance of weight averaging in pretraining, and motivate further study of horizon-free optimization.

\section*{Acknowledgements}
AM would like to thank Jacob Zavatone-Veth, Alex Damian, Jingfeng Wu and Mary Letey for helpful discussions. The authors would also like to thank Max Shad and Bala Desinghu for their help with the cluster. AM, DM, PN acknowledge the support of a Kempner Institute Graduate Research Fellowship. CP
is supported by an NSF CAREER Award (IIS-2239780), DARPA grants DIAL-FP-038 and AIQ-HR00112520041, the
Simons Collaboration on the Physics of Learning and Neural Computation, and the William F. Milton Fund from
Harvard University. AM, SK, DM and PN
acknowledge that this work has been made possible in part by a gift from the Chan Zuckerberg
Initiative Foundation to establish the Kempner Institute for the Study of Natural and Artificial
Intelligence. SK and DM acknowledge support from the Office of Naval Research under award N0001422-1-2377 and the National Science Foundation Grant under award \#IIS 2229881. DM is also supported by a Simons Investigator Fellowship, NSF grant DMS-2134157, DARPA grant W911NF2010021,and DOE grant DE-SC0022199.

\bibliography{references}
\bibliographystyle{unsrtnat}

\newpage
\appendix
\onecolumn

\section{Proofs of Section~\ref{sec:theoretical_analysis}}
\label{app:proofs}
\paragraph{Helper lemmas.} Before beginning the main proofs, we state a few helper lemmas. Also note that for the majority of the proofs, we will make heavy use of the inequality:
\begin{align}
    (1-x)^t \leq \exp(-tx)
\end{align}
\begin{lemma}
\label{lem:mvt_bound}
Let $0<\alpha\le \frac12$ and let $1\le j\le i$. Then
\[
i^{-\alpha}(i-j)
\;\lesssim\;
i^{1-\alpha}-j^{1-\alpha}
\;\lesssim\;
j^{-\alpha}(i-j).
\]
\end{lemma}

\begin{proof}
Let $f(x)=x^{1-\alpha}$. By the mean value theorem applied on
$[j,i]$, there exists $c\in[j,i]$ such that
\[
i^{1-\alpha}-j^{1-\alpha}
=
(1-\alpha)c^{-\alpha}(i-j).
\]
Because $x\mapsto x^{-\alpha}$ is decreasing,
\[
(1-\alpha)i^{-\alpha}(i-j)
\le
i^{1-\alpha}-j^{1-\alpha}
\le
(1-\alpha)j^{-\alpha}(i-j).
\]
Since $0<\alpha\le \frac12$, we have $\frac{1}{2} \leq 1-\alpha \leq 1$ which proves the claim.
\end{proof}
\begin{lemma}
    \label{lem:bounded_risk}
    Let $\LLambda = \diag(\lambda_i)$ be a diagonal matrix of
    nonnegative eigenvalues, let $\lambda$ be the vector corresponding
    to its diagonal, and suppose that $(\eta_t)_{t\geq 1}$ is
    nonincreasing and satisfies
    \[
    0 < \eta_t \leq \eta_1 \leq \frac{1}{2 \tr(\H)}.
    \]
    Then, for the recurrence
    \[
    \m_t
    =
    \left(
        \I
        - 2 \eta_t \LLambda
        + 2 \eta_t^2 \LLambda^2
        + \eta_t^2 \lambda \lambda^\top
    \right)\m_{t-1}
    + \eta_t^2 \sigma^2 \lambda,
    \qquad t\geq 1,
    \]
    with $\m_0\geq 0$ componentwise, we have
    \[
    \lambda^\top \m_t
    \leq
    2 \lambda^\top \m_0 + \sigma^2
    \qquad \forall t \geq 0.
    \]
\end{lemma}
\begin{proof}
    We first denote the diagonal matrices
    \[
    \D_t
    =
    \I
    - 2 \eta_t \LLambda
    + 2 \eta_t^2 \LLambda^2.
    \]
    We rewrite the recurrence as
    \[
    \m_t
    =
    \D_t \m_{t-1}
    + \eta_t^2 \lambda
    (\lambda^\top \m_{t-1} + \sigma^2)
    \]
    Unrolling the recurrence gives
    \[
    \m_t
    =
    \left(\prod_{r=1}^t \D_r\right)\m_0
    +
    \sum_{s=1}^t
    \eta_s^2
    \left(\prod_{r=s+1}^t \D_r\right)
    \lambda
    (\lambda^\top \m_{s-1} + \sigma^2),
    \]
    where an empty product is understood to be $\I$.

    Dotting with $\lambda$, we have
    \[
    \lambda^\top \m_t
    =
    \lambda^\top
    \left(\prod_{r=1}^t \D_r\right)\m_0
    +
    \sum_{s=1}^t
    \eta_s^2
    \lambda^\top
    \left(\prod_{r=s+1}^t \D_r\right)
    \lambda
    (\lambda^\top \m_{s-1} + \sigma^2).
    \]
    First, note that for each diagonal entry of $\D_t$, we have
    \[
    0
    \leq
    \D_{t,i}
    =
    1 - 2\eta_t\lambda_i + 2\eta_t^2\lambda_i^2
    =
    (1-\eta_t\lambda_i)^2 + \eta_t^2\lambda_i^2
    \leq 1.
    \]
    This implies that, elementwise,
    \begin{align}
        \label{eqn:first_thing}
        \lambda^\top
        \left(\prod_{r=1}^t \D_r\right)\m_0
        \leq
        \lambda^\top \m_0.
    \end{align}

    Next, we look at the other term. Since
    \[
    1-\D_{s,i}
    =
    2\eta_s\lambda_i(1-\eta_s\lambda_i),
    \]
    and $\eta_s\lambda_i\leq 1/2$, we have
    \[
    \eta_s^2\lambda_i^2
    =
    \frac{\eta_s\lambda_i}
         {2(1-\eta_s\lambda_i)}
    (1-\D_{s,i})
    \leq
    \eta_1\lambda_i(1-\D_{s,i}).
    \]
    Therefore,
    \begin{align}
        \label{eqn:second_thing}
        &
        \sum_{s=1}^t
        \eta_s^2
        \lambda^\top
        \left(\prod_{r=s+1}^t \D_r\right)
        \lambda
        \nonumber\\
        &\leq
        \eta_1
        \sum_i \lambda_i
        \sum_{s=1}^t
        (1-\D_{s,i})
        \prod_{r=s+1}^t \D_{r,i}
        \nonumber\\
        &=
        \eta_1
        \sum_i \lambda_i
        \left(
            1-\prod_{r=1}^t \D_{r,i}
        \right)\\
        &\leq
        \eta_1 \tr(\H)
        \leq
        \frac{1}{2}.
    \end{align}

    We prove the final statement inductively. Define the hypothesis
    \[
    P(t):
    \lambda^\top \m_t
    \leq
    2\lambda^\top \m_0 + \sigma^2.
    \]

    The base case $P(0)$ is trivially true. Now suppose $P(k)$ is true
    for all $k<t$. We want to prove $P(t)$. We have
    \begin{align*}
        \lambda^\top \m_t
        &=
        \lambda^\top
        \left(\prod_{r=1}^t \D_r\right)\m_0
        +
        \sum_{s=1}^t
        \eta_s^2
        \lambda^\top
        \left(\prod_{r=s+1}^t \D_r\right)
        \lambda
        (\lambda^\top \m_{s-1}+\sigma^2)
        \\
        &\leq
        \lambda^\top \m_0
        +
        \sum_{s=1}^t
        \eta_s^2
        \lambda^\top
        \left(\prod_{r=s+1}^t \D_r\right)
        \lambda
        (\lambda^\top \m_{s-1}+\sigma^2)
        && \eqref{eqn:first_thing}
        \\
        &\leq
        \lambda^\top \m_0
        +
        (2\lambda^\top \m_0+2\sigma^2)
        \sum_{s=1}^t
        \eta_s^2
        \lambda^\top
        \left(\prod_{r=s+1}^t \D_r\right)
        \lambda
        && P(s-1)
        \\
        &\leq
        2\lambda^\top \m_0+\sigma^2
        && \eqref{eqn:second_thing}.
    \end{align*}
    This proves $P(t)$ and completes the induction.
\end{proof}

\paragraph{Setup.} Before proceeding to the proofs, we briefly restate the notation and setup from Section~\ref{sec:theoretical_analysis}. Denote the independent and identically distributed input output pairs $\{(\x_i, y_i)\}_{i \geq 1}$ where each $\x_i$ and $y_i$ are $\x_i \sim \N(0, \H)$ and $y_i = \x_i^\top w^\star + \epsilon_i$ for independent $\epsilon_i \sim \N(0, \sigma^2)$. We take $\w_0=0$ as the initialization and number updates from $t=1$, so update $t$ uses $(\x_t,y_t)$ and produces $\w_t$. Here, $s\geq1$ is the step from which we start tail averaging. Define the risk to be $\risk(\w) = \frac{1}{2} \E_{\x, y} (\la \w, \x \ra - y)^2$ where the expectation is over the joint $(\x, y) \sim \mathcal{D}$. Consider the learning rate schedule
\begin{align*}
    \eta_t = \frac{\eta}{t^\gamma} \qquad t \geq 1
\end{align*}
for $\gamma \in (0, 1/2]$ and a constant $0 < \eta \leq 1/{2\tr(\H)}$. Denote the eigendecomposition of $\H = \Q \LLambda \Q^\top$ and let $\SSigma_i = \E [(\w_i - \w^\star)(\w_i - \w^\star)^\top]$ be the second moment of the iterates. We also introduce $\M_t = \Q^\top \SSigma_t \Q$ as the covariance of the iterates rotated in the eigenbasis of the data, and $\m_t = \diag(\M_t)$ as its diagonal. Note that for $\mathbf{A}, \mathbf{B}$ matrices we have $\la \mathbf{A}, \mathbf{B} \ra = \tr(\mathbf{A}^\top \mathbf{B})$. We also denote the induced norm $\|\w\|_{\mathbf{A}}^2 = \w^\top \mathbf{A}\w$. For a diagonal matrix we define its restriction to entries between $i$ and $j$ where $1 \leq i \leq j \leq \infty$ as $\LLambda_{i:j} = \diag(\lambda_i, \dots, \lambda_j)$. Notationally, for $\w$ and $\v$ multidimensional quantities (vectors or matrices), we mean $\w \leq \v$ as componentwise comparison, and we will use $\preceq$ to denote the Loewner order comparisons.

\subsection{Proof of Theorem~\ref{thm:tgamma_rate}}
\label{sec:proof_thm1}
\begin{proof}[Proof of Theorem~\ref{thm:tgamma_rate}]
We begin with deriving the expression for the risk:
\begin{align*}
    \risk(\bw_{s:s+N})
    &= \frac{1}{2}\E\left[\big(\la \bw_{s:s+N} - \w^\star, \x \ra - \epsilon\big)^2\right].
\end{align*}
Expanding and using $\E[\epsilon]=0$,
\begin{align*}
    \risk(\bw_{s:s+N})
    &= \frac{1}{2}\E\left[\la \bw_{s:s+N} - \w^\star, \x \ra^2\right]
    + \frac{\sigma^2}{2}.
\end{align*}
Substituting $\bw_{s:s+N} = \frac{1}{N}\sum_{i=s}^{s+N-1}\w_i$,
\begin{align*}
    \risk(\bw_{s:s+N})
    &= \frac{1}{2N^2}\sum_{i=s}^{s+N-1}\sum_{j=s}^{s+N-1}
    \E\left[
        \x^\top(\w_i-\w^\star)(\w_j-\w^\star)^\top \x
    \right]
    + \frac{\sigma^2}{2} \\
    &\le \frac{1}{N^2}\sum_{i=s}^{s+N-1}\sum_{j=i}^{s+N-1}
    \E\left[
        \x^\top(\w_i-\w^\star)(\w_j-\w^\star)^\top \x
    \right]
    + \frac{\sigma^2}{2}.
\end{align*}

To make progress, we invoke the tower rule:
\begin{align*}
    \E\left[(\w_j-\w^\star)(\w_i-\w^\star)^\top\right]
    &= \E\left[\E[\w_j-\w^\star \mid \w_i]\;(\w_i-\w^\star)^\top\right].
\end{align*}

For SGD with stepsizes $\eta_t$ we have:
\begin{align*}
    \E[\w_j-\w^\star \mid \w_i]
    &= \prod_{t=i+1}^{j}(\I-\eta_t \H)(\w_i-\w^\star)
\end{align*}
Since $\eta \lambda_i \leq 1$ for all $\lambda_i$ we have the PSD inequality (the factors commute since they have the same $\H$ basis):
\[
\prod_{t=i+1}^{j}(\I-\eta_t \H) \preceq \exp\ \qty(-\H\sum_{t=i+1}^{j}\eta_t) \preceq \exp\qty(-\eta [j^{1-\gamma} - i^{1-\gamma}]\H)
\]
where the last inequality is from Lemma~\ref{lem:mvt_bound}. Plugging this inside of the risk expression and diagonalizing we have the following risk inequality:
\begin{align}
    \label{eq:risk}
    \risk(\bw_{s:s+N}) - \frac{\sigma^2}{2}
    \lesssim
    \frac{1}{N^2}\sum_{i=s}^{s+N-1}\sum_{j=i}^{s+N-1}
    \la \m_i,\ \exp\Big(-\eta(j^{1-\gamma}-i^{1-\gamma})\LLambda\Big)\lambda \ra
\end{align}
To get an expression for the bias and variance iterates, denoted as $\bm_t$ and $\vm_t$ respectively, we follow the derivation from~\citet{meterez2025simplified}. At batch size $1$, the SGD update is
\begin{align*}
    \w_t
    &= \w_{t-1} - \eta_t \x_t(\x_t^\top \w_{t-1} - y_t),
    \qquad t\geq1,
    \qquad
    y_t = \x_t^\top \w^\star + \epsilon_t ,
\end{align*}
where $\epsilon_t$ is independent noise with $\E[\epsilon_t]=0$ and
$\E[\epsilon_t^2]=\sigma^2$.
Subtracting $\w^\star$,
\begin{align*}
    \w_t-\w^\star
    &= (\I-\eta_t \x_t\x_t^\top)(\w_{t-1}-\w^\star)
    + \eta_t \epsilon_t \x_t .
\end{align*}
Taking the covariance of this term we obtain a recursion on $\SSigma_t$:
\begin{align*}
    \SSigma_t
    &\coloneqq \E[(\w_t-\w^\star)(\w_t-\w^\star)^\top] \\
    &=
    \E\left[
        (\I-\eta_t \x\x^\top)\SSigma_{t-1}(\I-\eta_t \x\x^\top)
    \right]
    + \eta_t^2\sigma^2 \E[\x\x^\top].
\end{align*}
Taking expectation yields:
\begin{align*}
    \SSigma_t = \SSigma_{t-1} - \eta_t \SSigma_{t-1} \H - \eta_t \H \SSigma_{t-1} + 2 \eta_t^2 \H \SSigma_{t-1} \H + \eta_t^2 \tr(\H \SSigma_{t-1}) \H + \eta_t^2 \sigma^2 \H
\end{align*}
Rotating in the $\Q$ basis and taking a diagonal operator through the whole equation we end up with a recursion on $\m_t$:
\begin{align*}
    \m_t
    &=
    \left(
        \I
        - 2\eta_t\LLambda
        + 2\eta_t^2\LLambda^2
        + \eta_t^2\lambda\lambda^\top
    \right)\m_{t-1}
    + \sigma^2\eta_t^2\lambda \\
    &\leq \left(
        \I
        - 2\eta_t\LLambda
        + \eta_t^2\LLambda^2
        + 2\eta_t^2\lambda\lambda^\top
    \right)\m_{t-1}
    + \sigma^2\eta_t^2\lambda && \LLambda^2 \leq \lambda \lambda^\top \\
    &\leq
    (\I-\eta_t\LLambda)^2\m_{t-1}
    + 2\eta_t^2 \lambda (2 \lambda^\top \m_0 + \sigma^2) + \sigma^2\eta_t^2\lambda
    && \text{Lemma~\ref{lem:bounded_risk}} \\
    &\lesssim (\I-\eta_t\LLambda)^2\m_{t-1} + 
    7 \sigma^2\eta_t^2\lambda && \|\w_0 - \w^\star\|_\H^2 \lesssim \sigma^2
\end{align*}
For simplicity, we will absorb the constant $7$ into the noise scale $\sigma^2$ for the remainder of this proof. Unrolling the recursion yields
\begin{align}
    \m_t
    &\leq
    \Bigg[\prod_{i=1}^t (I-\eta_i \LLambda)^2\Bigg]\m_0
    + \sum_{p=1}^t \eta_p^2\sigma^2
    \Bigg[\prod_{r=p+1}^t (I-\eta_r \LLambda)^2\Bigg]\lambda \\
    &\le
    \exp\Big[-2\LLambda \sum_{i=1}^t \eta_i\Big]\m_0
    + \sigma^2 \sum_{p=1}^t \eta_p^2
    \exp\Big[-2\LLambda \sum_{r=p+1}^t \eta_r\Big]\lambda
    \label{eq:bias_and_variance_eta}
\end{align}
Plugging in the scheduler for $\eta_t$, we have:
\begin{align*}
    \m_t
    &\lesssim
    \exp\big[-2\eta \LLambda t^{1-\gamma}\big]\m_0
    + \eta^2\sigma^2 \sum_{p=1}^t \frac{1}{p^{2\gamma}}
    \exp\Big[-2\eta\LLambda (t^{1-\gamma}-p^{1-\gamma})\Big]\lambda .
\end{align*}

We therefore define the bias and variance iterates as
\begin{align}
    \bm_t
    &\coloneqq
    \exp\big[-2\eta \LLambda t^{1-\gamma}\big]\m_0, \\
    \vm_t
    &\coloneqq
    \eta^2\sigma^2 \sum_{p=1}^t \frac{1}{p^{2\gamma}}
    \exp\Big[-2\eta\LLambda (t^{1-\gamma}-p^{1-\gamma})\Big]\lambda .
\end{align}
To obtain the rates, we independently bound each of these quantities after plugging them back in~\eqref{eq:risk}. 
\paragraph{Bias bound.} We begin with the bias:
\begin{align*}
    \bias_t
    &\lesssim
    \sum_{k} \frac{\m_{0,k}\lambda_k}{N^2}
    \sum_{i=s}^{s+N-1} \sum_{j=i}^{s+N-1}
    \exp\Big[-\eta(j^{1-\gamma}-i^{1-\gamma})\lambda_k
               -2\eta \lambda_ki^{1-\gamma}\Big] \\
    &\le
    \sum_{k} \frac{\m_{0,k}\lambda_k}{N^2}
    \sum_{i=s}^{s+N-1} \sum_{j=i}^{s+N-1}
    \exp\Big[-\eta(j^{1-\gamma}-i^{1-\gamma})\lambda_k
               -\eta \lambda_k i^{1-\gamma}\Big] \\
    &=
    \sum_{k} \frac{\m_{0,k}\lambda_k}{N^2}
    \sum_{i=s}^{s+N-1} \sum_{j=i}^{s+N-1}
    \exp\big[-\eta \lambda_k j^{1-\gamma}\big].
\end{align*}

Choose
\[
k^\star
\coloneqq
\max\Big\{k:\ \lambda_k \ge \tfrac{\log N}{\eta s^{1-\gamma}}\Big\},
\]
which defines the split between head and tail eigenvalues.

\begin{align*}
    \bias_t^{1:k^\star}
    &\lesssim
    \sum_{k \le k^\star} \frac{\m_{0,k}\lambda_k}{N^2}
    \sum_{i=s}^{s+N-1} \sum_{j=i}^{s+N-1}
    \exp\big[-\eta \lambda_k j^{1-\gamma}\big] \\
    &\le
    \sum_{k \le k^\star} \frac{\m_{0,k}\lambda_k}{N^2}
    \sum_{i=s}^{s+N-1} \sum_{j=i}^{s+N-1}
    \exp\Big[-\Big(\tfrac{j}{s}\Big)^{1-\gamma}\log N\Big] \\
    &\lesssim
    \sum_{k \le k^\star} \frac{\m_{0,k}\lambda_k}{N^2}
    \sum_{i=s}^{s+N-1} \sum_{j=i}^{s+N-1} N^{-1} \\
    &\lesssim
    \sum_{k \le k^\star} \frac{\m_{0,k}\lambda_k}{N}.
\end{align*}

For the tail eigenvalues, we upper bound the exponential by \(1\), yielding
\begin{align*}
    \bias_t^{k^\star:\infty}
    &\lesssim
    \sum_{k > k^\star} \m_{0,k}\lambda_k .
\end{align*}

Setting $\w_0 = 0$, we have the final bias bound:
\begin{align}
    \bias_t \lesssim \frac{1}{N} \|\w^\star\|_{\LLambda_{1:k^\star}}^2 + \|\w^\star\|_{\LLambda_{k^\star:\infty}}^2
\end{align}

\paragraph{Variance bound.} We now turn our attention to bounding the variance. Following a similar procedure as before, we have:
\begin{align*}
    \variance_t &\lesssim \sum_{k} \frac{\eta^2 \sigma^2 \lambda_k^2}{N^2} \sum_{i=s}^{s+N-1} \sum_{j=i}^{s+N-1} \sum_{p=1}^{i} \frac{1}{p^{2\gamma}} \exp\qty[-2\eta \lambda_k (i^{1-\gamma} - p^{1-\gamma}) -\eta \lambda_k(j^{1-\gamma} - i^{1-\gamma})] \\
    &\leq \sum_{k}\frac{\eta^2 \sigma^2 \lambda_k^2}{N^2} \sum_{i=s}^{s+N-1} \sum_{j=i}^{s+N-1} \sum_{p=1}^{i} \frac{1}{p^{2\gamma}} \exp\qty[-\eta \lambda_k (j^{1-\gamma} - p^{1-\gamma})] \\
    &= \sum_{k}\frac{\eta^2 \sigma^2 \lambda_k^2}{N^2} \sum_{i=s}^{s+N-1} \sum_{j=i}^{s+N-1} \sum_{p=1}^{i} \frac{1}{p^{2\gamma}} \exp\qty[-\eta \lambda_k (j^{1-\gamma} - i^{1-\gamma} + i^{1-\gamma} - p^{1-\gamma})] \\
    &= \sum_{k}\frac{\eta^2 \sigma^2 \lambda_k^2}{N^2} \sum_{i=s}^{s+N-1} \sum_{j=i}^{s+N-1} \exp\qty[-\eta\lambda_k (j^{1-\gamma} - i^{1-\gamma})]\sum_{p=1}^{i} \frac{1}{p^{2\gamma}} \exp\qty[-\eta \lambda_k(i^{1-\gamma} - p^{1-\gamma})]
\end{align*}

Now we can handle the last sum. We split it at $i/2$ and handle each sum independently. We begin with the first half:

\begin{align*}
    \sum_{p=1}^{i/2} \frac{1}{p^{2\gamma}} \exp\qty[-\eta \lambda_k(i^{1-\gamma} - p^{1-\gamma})] &\lesssim \sum_{p=1}^{i/2} \frac{1}{p^{2\gamma}} \exp\qty[-\eta \lambda_k i^{-\gamma}(i-p)] & \text{Lemma~\ref{lem:mvt_bound}} \\
    &\lesssim \exp\qty(-\frac{1}{2}\eta \lambda_k i^{1-\gamma}) \sum_{p=1}^{i/2} \frac{1}{p^{2\gamma}} \\
    &\lesssim \exp\qty(-\frac{1}{2}\eta \lambda_k i^{1-\gamma}) \cdot
        \begin{cases}
    i^{1 - 2\gamma} & \gamma < 1/2\\
    \log{i} & \gamma=1/2 
    \end{cases}
\end{align*}

Now we look at the second half. We have:
\begin{align*}
     \sum_{p=i/2}^{i} \frac{1}{p^{2\gamma}} \exp\qty[-\eta \lambda_k(i^{1-\gamma} - p^{1-\gamma})] &\lesssim i^{-2\gamma} \exp\qty(-\eta \lambda_k i^{1-\gamma}) \sum_{p=i/2}^i \exp\qty(\eta \lambda_k i^{-\gamma}p) \\
     &\lesssim \frac{i^{-\gamma}}{\eta \lambda_k} \qty[1 - \exp\qty(-\frac{1}{2} \eta \lambda_k i^{1-\gamma})]
\end{align*}

Now that we have obtained the bounds for the sum over $p$, remains to bound the sum over $j$:
\begin{align*}
     \sum_{j=i}^{s+N-1} \exp\qty[-\eta\lambda_k (j^{1-\gamma} - i^{1-\gamma})] &\lesssim \sum_{j=i}^{s+N-1} \exp\qty[-\eta\lambda_k j^{-\gamma}(j-i)] & \text{Lemma~\ref{lem:mvt_bound}}\\
     &\lesssim \sum_{j=i}^{s+N-1} \exp\qty[-\eta\lambda_k (s+N)^{-\gamma}(j-i)] \\
     &\leq \sum_{q=0}^{N-1} \exp\qty[-\eta\lambda_k (s+N)^{-\gamma}q] \\
     &\lesssim \frac{(s+N)^\gamma}{\eta \lambda_k} \qty(1 - \exp\qty[-\eta \lambda_k \frac{N}{(s+N)^\gamma}])
\end{align*}

Assembling everything together we end up with:
\begin{align*}
\variance_t \lesssim\; &
\sum_{k}\frac{\eta^2 \sigma^2 \lambda_k^2}{N^2}
\sum_{i=s}^{s+N-1} \sum_{j=i}^{s+N-1}
\exp\left[-\eta\lambda_k \left(j^{1-\gamma} - i^{1-\gamma}\right)\right]
\sum_{p=1}^{i} \frac{1}{p^{2\gamma}}
\exp\left[-\eta \lambda_k\left(i^{1-\gamma} - p^{1-\gamma}\right)\right] \\
\lesssim\; &
\sum_{k}\frac{\eta^2 \sigma^2 \lambda_k^2}{N^2}
\sum_{i=s}^{s+N-1}\frac{(s+N)^\gamma}{\eta \lambda_k}
\left(1 - \exp\left[-\eta \lambda_k \frac{N}{(s+N)^\gamma}\right]\right)
\Bigg[
\exp\left(-\frac{1}{2}\eta \lambda_k i^{1-\gamma}\right) s_i(\gamma)
\\&+ \frac{i^{-\gamma}}{\eta \lambda_k}
\left(1 - \exp\left(-\frac{1}{2} \eta \lambda_k i^{1-\gamma}\right)\right)
\Bigg].
\end{align*}

where we have defined:
\begin{align*}
    s_i(\gamma) = \begin{cases}
    i^{1 - 2\gamma} & \gamma < 1/2\\
    \log{i} & \gamma=1/2 
    \end{cases}
\end{align*}

Recall the cutoff
\[
k^\star \coloneqq \max\Big\{k:\ \lambda_k \ge \tfrac{\log N}{\eta s^{1-\gamma}}\Big\}.
\]
We split the sum over \(k\) into head and tail components.

\paragraph{Variance head.}
For \(k \le k^\star\) and all \(i \ge s\), we have
\[
\eta \lambda_k i^{1-\gamma} \ge \eta \lambda_k s^{1-\gamma} \ge \log N,
\]
so the exponential term in the \(p\)-sum decays rapidly. Consequently, the sum
over \(p\) is dominated by its second half, yielding the bound
\[
\sum_{p=1}^{i} \frac{1}{p^{2\gamma}}
\exp\Big[-\eta \lambda_k(i^{1-\gamma} - p^{1-\gamma})\Big]
\;\lesssim\;
\frac{i^{-\gamma}}{\eta \lambda_k}.
\]
Moreover,
\(
1 - \exp[-\eta \lambda_k \tfrac{N}{(s+N)^\gamma}] \le 1
\).
Thus, we get:
\begin{align*}
\variance_t^{1:k^\star}
&\lesssim
\sum_{k\le k^\star}
\frac{\eta^2 \sigma^2 \lambda_k^2}{N^2}
\sum_{i=s}^{s+N-1}
\frac{(s+N)^\gamma}{\eta \lambda_k}
\cdot
\frac{i^{-\gamma}}{\eta \lambda_k} \\
&=
\sigma^2
\sum_{k\le k^\star}
\frac{(s+N)^\gamma}{N^2}
\sum_{i=s}^{s+N-1} i^{-\gamma}.
\end{align*}
Using \(\sum_{i=s}^{s+N-1} i^{-\gamma} \lesssim N s^{-\gamma}\), we conclude
\begin{align*}
\variance_t^{1:k^\star}
&\lesssim
\sigma^2
\frac{(s+N)^\gamma}{N s^\gamma}\, k^\star
\label{eq:var_head_clean}
\end{align*}

\paragraph{Variance tail.}
For \(k > k^\star\), we upper bound all exponentials by \(1\) and use
\(1-e^{-x}\le x\) with \(x=\eta\lambda_k \tfrac{N}{(s+N)^\gamma}\). Starting from the assembled bound,
\begin{align*}
\variance_t
&\lesssim
\sum_{k}\frac{\eta^2 \sigma^2 \lambda_k^2}{N^2} \sum_{i=s}^{s+N-1}
\frac{(s+N)^\gamma}{\eta \lambda_k}
\qty(1 - \exp\qty[-\eta \lambda_k \tfrac{N}{(s+N)^\gamma}])\\
&\qty[\exp\qty(-\tfrac{1}{2}\eta \lambda_k i^{1-\gamma}) s_i(\gamma)
+ \frac{i^{-\gamma}}{\eta \lambda_k} \qty[1 - \exp\qty(-\tfrac{1}{2} \eta \lambda_k i^{1-\gamma})]] ,
\end{align*}
we obtain, for \(k>k^\star\),
\begin{align*}
\variance_t^{k^\star:\infty}
&\lesssim
\sum_{k>k^\star}\frac{\eta^2 \sigma^2 \lambda_k^2}{N^2} \sum_{i=s}^{s+N-1}
\frac{(s+N)^\gamma}{\eta \lambda_k}
\qty(\eta \lambda_k \tfrac{N}{(s+N)^\gamma})
\qty[s_i(\gamma) + \frac{i^{-\gamma}}{\eta \lambda_k}] \\
&=
\sum_{k>k^\star}\frac{\eta^2 \sigma^2 \lambda_k^2}{N} \sum_{i=s}^{s+N-1}
\qty[s_i(\gamma) + \frac{i^{-\gamma}}{\eta \lambda_k}] \\
&=
\sigma^2
\sum_{k>k^\star}
\left(
    \eta^2 \lambda_k^2\sum_{i=s}^{s+N-1} s_i(\gamma)
    + \eta \lambda_k \sum_{i=s}^{s+N-1} i^{-\gamma}
\right)\frac{1}{N}.
\end{align*}
Bounding the remaining sums as
\(
\frac{1}{N}\sum_{i=s}^{s+N-1} i^{-\gamma}\lesssim s^{-\gamma}
\)
and
\(
\frac{1}{N}\sum_{i=s}^{s+N-1} s_i(\gamma)\lesssim s^{1-2\gamma}
\)
(for \(\gamma=\tfrac12\) this incurs an additional \(\log N\) factor),
we conclude
\begin{align}
\variance_t^{k^\star:\infty}
\lesssim
\sigma^2
\sum_{k>k^\star}
\left(
    \eta^2 \lambda_k^2\, s^{1-2\gamma}
    + \eta \lambda_k\, s^{-\gamma}
\right)
\end{align}
Combining the 2 bounds we get the final variance bound:
\begin{align}
    \variance_t \lesssim \sigma^2
\frac{(s+N)^\gamma}{N s^\gamma}\, k^\star + \sigma^2
\sum_{k>k^\star}
\left(
    \eta^2 \lambda_k^2\, s^{1-2\gamma}
    + \eta \lambda_k\, s^{-\gamma}
\right)
\end{align}
\end{proof}
\subsection{Proof of Corollary 1}
\begin{proof}[Proof of Corollary~\ref{cor:optimal_hps}]
    From Theorem~\ref{thm:tgamma_rate} we know that for $s = \Theta(N)$, for $k^\star
    \coloneqq
    \max\Big\{k:\ \lambda_k \ge \tfrac{\log N}{\eta N^{1-\gamma}}\Big\}$ the bias and variance decay rates are:
    \begin{align*}
        \bias_t &\lesssim \frac{1}{N} \|\w^\star\|_{\LLambda_{1:k^\star}}^2 + \|\w^\star\|_{\LLambda_{k^\star:\infty}}^2\\
        \variance_t &\lesssim \frac{\sigma^2 k^\star}{N}
        + \sigma^2
        \sum_{k>k^\star}
        \left(
            \eta^2 \lambda_k^2\, N^{1-2\gamma}
            + \eta \lambda_k\, N^{-\gamma}
        \right)
    \end{align*}

    We first specialize the bias bound to the power law setting. Plugging in the source and capacity assumptions, we have:
    \begin{align*}
        \|\w^\star\|_{\LLambda_{1:k^\star}}^2 &\leq \sum_{k\leq k^\star} \lambda_k (w_k^\star)^2 \lesssim \sum_{k\leq k^\star} k^{-b} \lesssim \sum_{k = 1}^\infty k^{-b} \eqsim 1 \qquad b > 1\\
        \|\w^\star\|_{\LLambda_{k^\star:\infty}}^2 &\lesssim \sum_{k>k^\star} k^{-b} \leq \int_{k^\star}^\infty k^{-b} dk\lesssim (k^\star)^{1-b}
    \end{align*}

    For the variance term, we can see that the variance head is the dominating term. By definition, we have  $\lambda_k \eqsim k^{-a}$ with $a>1$. Then:
\[
\sum_{k>k^\star}\lambda_k \lesssim (k^\star)^{1-a},
\qquad
\sum_{k>k^\star}\lambda_k^2 \lesssim (k^\star)^{1-2a}.
\]
Moreover, by definition of $k^\star$ we have $\lambda_{k^\star}\eqsim (k^\star)^{-a}\eqsim \frac{\log{N}}{\eta N^{1-\gamma}}$, where the $\eqsim$ notation absorbs constants and $\log$ factors. Thus, for the 2 terms in the variance tail, we have:
\begin{align*}
\sigma^2\eta N^{-\gamma}\sum_{k>k^\star}\lambda_k
&\lesssim
\sigma^2\eta N^{-\gamma}(k^\star)^{1-a}
=
\frac{\sigma^2 k^\star}{N}\,\Big(\eta N^{1-\gamma}(k^\star)^{-a}\Big)
\lesssim
\frac{\sigma^2 k^\star}{N}\log{N} \lesssim \frac{\sigma^2 k^\star}{N}\\
\sigma^2\eta^2 N^{1-2\gamma}\sum_{k>k^\star}\lambda_k^2
&\lesssim
\sigma^2\eta^2 N^{1-2\gamma}(k^\star)^{1-2a}
=
\frac{\sigma^2 k^\star}{N}\,\Big(\eta N^{1-\gamma}(k^\star)^{-a}\Big)^2
\lesssim
\frac{\sigma^2 k^\star}{N} \log^{2}{N} \lesssim \frac{\sigma^2 k^\star}{N}
\end{align*}

where we absorbed the $\log$ factors into the $\lesssim$ notation. Plugging everything in, we have the bound on the risk:
\begin{align*}
    \E\risk(\bw_{s:s+N}) - \frac{\sigma^2}{2} \lesssim \frac{1}{N} + (k^\star)^{1-b} + \frac{\sigma^2 k^\star}{N}
\end{align*}

Optimizing the bound over $\gamma$ is equivalent to optimizing over $k^\star$, since $k^\star$ is a function of $\gamma$. We can optimize up to constants by balancing the $k^\star$ dependent terms in the bound:
\begin{align*}
    (k^\star)^{1-b} \eqsim \frac{\sigma^2 k^\star}{N}
    \implies
    (k^\star)^{b} \eqsim \frac{N}{\sigma^2}
    \implies
    k^\star \eqsim N^{1/b}
\end{align*}

Combining with the definition of $k^\star$ we have $\lambda_{k^\star} \eqsim (k^\star)^{-a} \eqsim \frac{\log{N}}{\eta N^{1-\gamma}}$. Plugging in the expression for $k^\star$ we have and dropping the $\log$ factor and the constant $\eta$ we have:
\begin{align*}
    N^{(1-\gamma)/a} \eqsim k^\star \eqsim N^{1/b}
    \implies
    \gamma^\star = 1 - \frac{a}{b} \in (0, \frac{1}{2}]
\end{align*}

for $1 < a < b \leq 2a$. Thus the optimal exponent is $\gamma^\star$ for the rate. This yields:
\[
(k^\star)^{1-b} \eqsim N^{-(b-1)/b}
\]

Since $b>1$, we have that $1/N \leq N^{-(b-1)/b}$, and thus the final rate is:
\[
\E\risk(\bw_{s:s+N}) - \frac{\sigma^2}{2} \lesssim N^{-(b-1)/b}
\]
\end{proof}
\clearpage

\clearpage
\section{Additional Figures}
\label{app:additional_figures}
\subsection{Figure~\ref{fig:main_figure_both} at optimal hyperparameters}
We provide the losses of Figure~\ref{fig:main_figure_both} at the optimal hyperparameters for every intermediate point. Note that this is not a single run - we plot the optimal loss (over our sweep) at each intermediate point from $1 \times$ to $32 \times$ (and $16\times$ respectively for the 300M model) Chinchilla and interpolate between the points.

\begin{figure*}[!htp]
    \centering
    \includegraphics[width=1.0\linewidth]{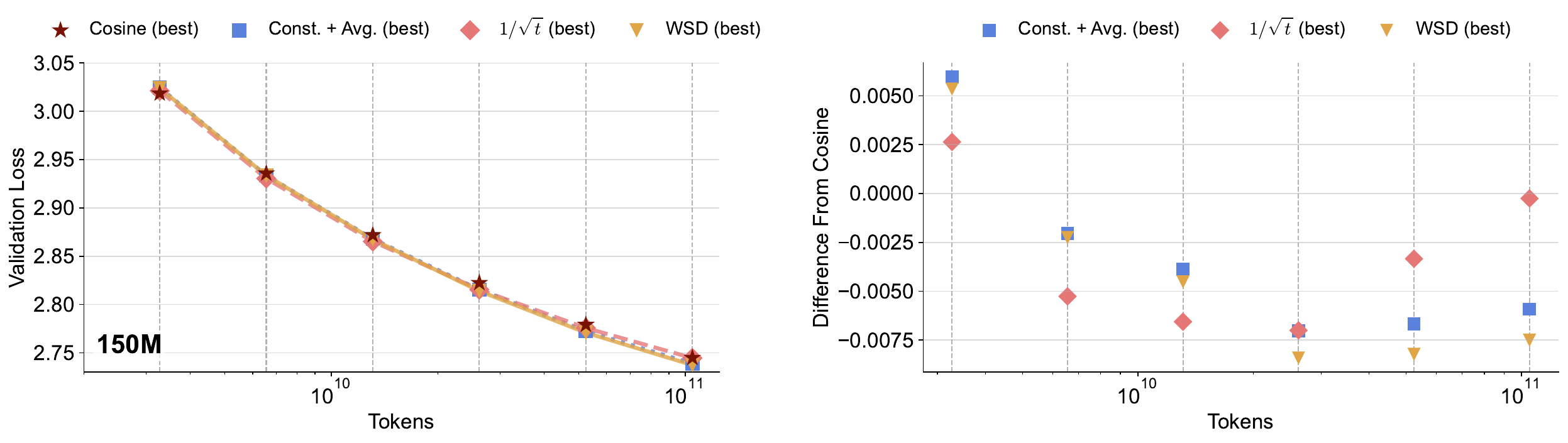}
    \vspace{0.5em}
    \includegraphics[width=1.0\linewidth]{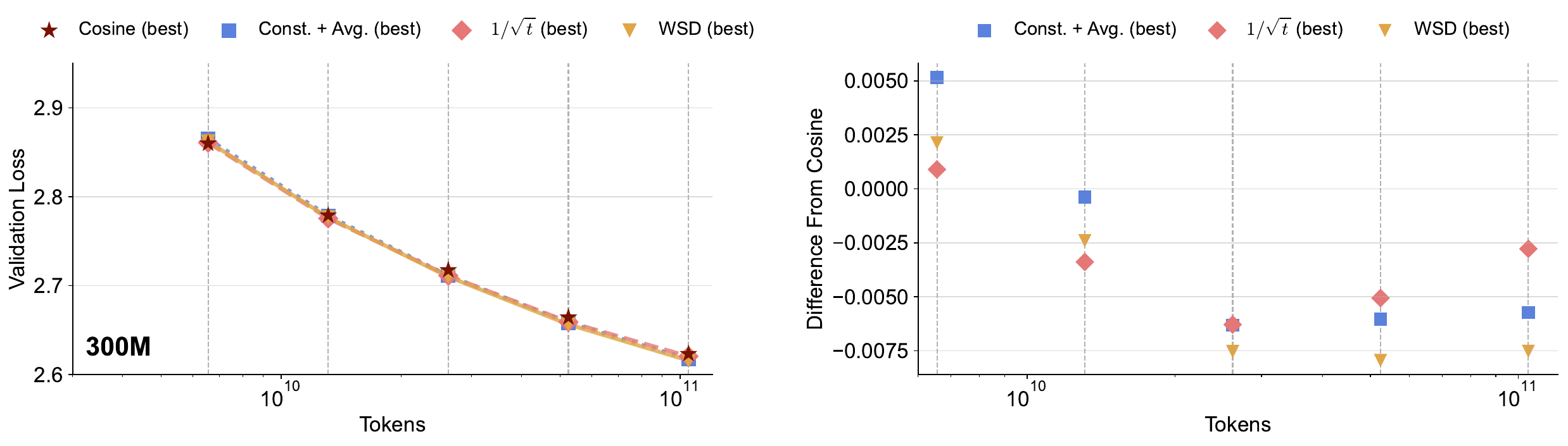}
  \caption{Top plots correspond to 150M parameter, and bottom plots correspond to 300M models. At each intermediate point, we plot the best loss out of the whole hyperparameter sweep, then we linearly interpolate between the points. Note that for long training durations, constant with averaging and $1/\sqrt{t}$ offers a substantial improvement over cosine decay.}
  \label{fig:main_fig_optimal_hps}
\end{figure*}

\clearpage
\subsection{Comparison at Optimal HPs}
\subsubsection{Optimal Hyperparameters by Endpoint}

\begin{table}[H]
    \centering
    \small
    \setlength{\tabcolsep}{9pt}
    \begin{tabular}{c cc ccc}
        \toprule
        & \multicolumn{2}{c}{Constant LR + Avg.} & \multicolumn{3}{c}{$1/\sqrt{t}$ + Avg.} \\
        \cmidrule(lr){2-3}\cmidrule(lr){4-6}
        Endpoint & $\eta$ & $\beta_2$ & $\eta$ & $\beta_2$ & $\alpha$ \\
        \midrule
        $1\times$  & 0.003 & 0.98 & 0.003 & 0.99 & 12,800 \\
        $2\times$  & 0.003 & 0.98 & 0.003 & 0.99 & 12,800 \\
        $4\times$  & 0.003 & 0.98 & 0.003 & 0.98 & 51,200 \\
        $8\times$  & 0.001 & 0.99 & 0.003 & 0.99 & 51,200 \\
        $16\times$ & 0.001 & 0.99 & 0.001 & 0.99 & 51,200 \\
        $32\times$ & 0.001 & 0.99 & 0.001 & 0.99 & 51,200 \\
        \bottomrule
    \end{tabular}
    \caption{Per-endpoint optimal hyperparameters for the 150M plots, where optimal means the lowest $4\%$ EMA validation loss among the configurations ran at each checkpoint. }
    \label{tab:optimal-hps-150m}
\end{table}

\begin{table}[H]
    \centering
    \small
    \setlength{\tabcolsep}{9pt}
    \begin{tabular}{c cc ccc}
        \toprule
        & \multicolumn{2}{c}{Constant LR + Avg.} & \multicolumn{3}{c}{$1/\sqrt{t}$ + Avg.} \\
        \cmidrule(lr){2-3}\cmidrule(lr){4-6}
        Endpoint & $\eta$ & $\beta_2$ & $\eta$ & $\beta_2$ & $\alpha$ \\
        \midrule
        $1\times$  & 0.003 & 0.95 & 0.003 & 0.95 & 12,800 \\
        $2\times$  & 0.003 & 0.95 & 0.003 & 0.95 & 12,800 \\
        $4\times$  & 0.001 & 0.99 & 0.001 & 0.99 & 51,200 \\
        $8\times$  & 0.001 & 0.99 & 0.001 & 0.99 & 51,200 \\
        $16\times$ & 0.001 & 0.99 & 0.001 & 0.99 & 51,200 \\
        \bottomrule
    \end{tabular}
    \caption{Per-endpoint optimal hyperparameters for the 300M plots, where optimal means the lowest $4\%$ EMA validation loss among the configurations ran at each checkpoint.}
    \label{tab:optimal-hps-300m}
\end{table}
\subsubsection{Inverse-Square-Root Offset Sensitivity}
\begin{figure*}[!htp]
    \centering
    \begin{subfigure}[t]{0.49\textwidth}
        \centering
        \includegraphics[width=\linewidth]{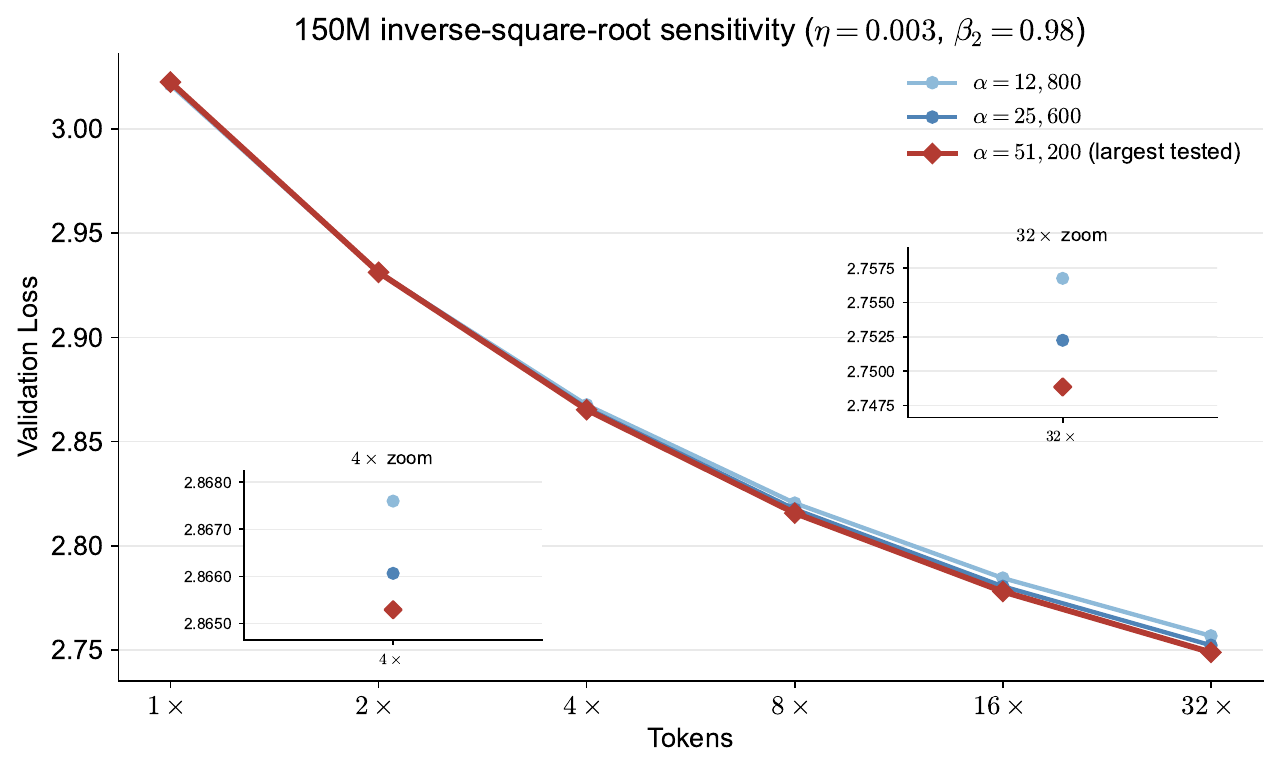}
    \end{subfigure}\hfill
    \begin{subfigure}[t]{0.49\textwidth}
        \centering
        \includegraphics[width=\linewidth]{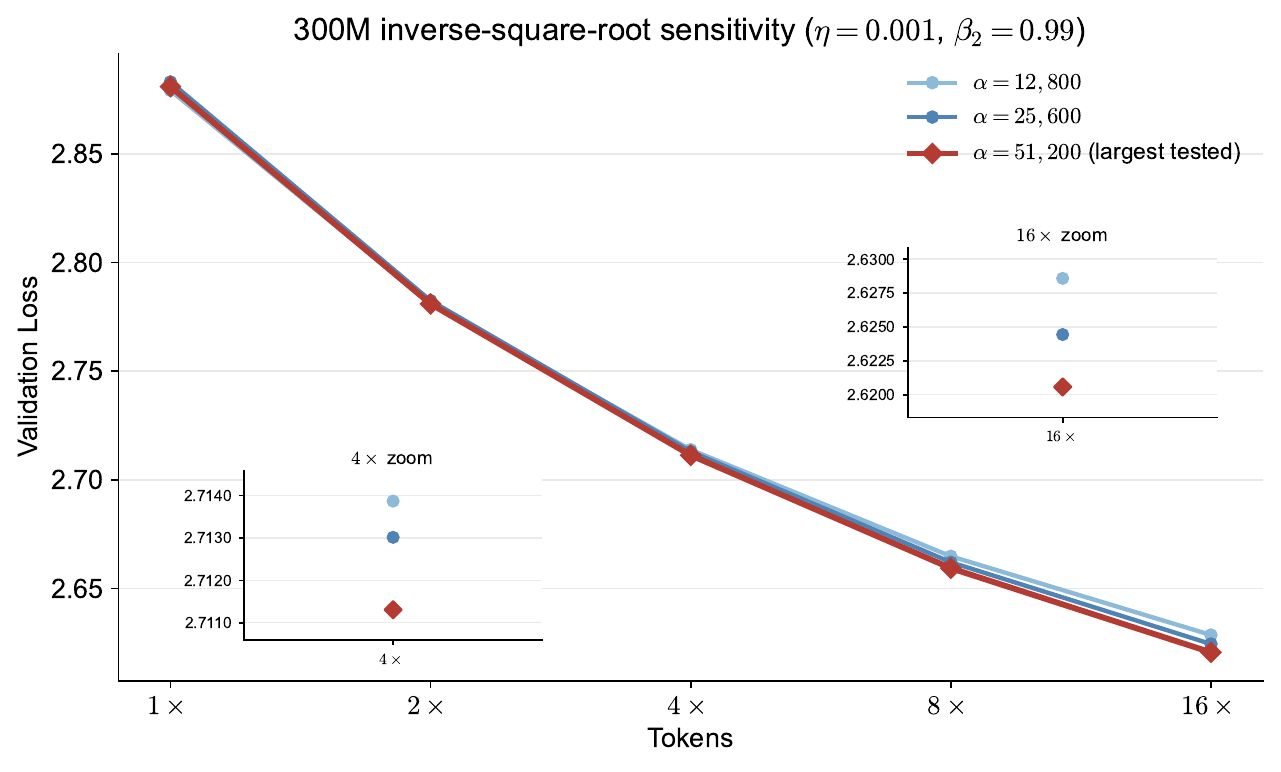}
    \end{subfigure}
    \caption{Sensitivity of the inverse-square-root offset $\alpha$ with every other hyperparameter fixed to the configuration selected at $4\times$. Left: the 150M model with $(\eta,\beta_2)=(0.003,0.98)$, evaluated from $1\times$ through $32\times$. Right: the 300M model with $(\eta,\beta_2)=(0.001,0.99)$, evaluated from $1\times$ through $16\times$. Lines connect exact checkpoint losses for $\alpha\in\{12{,}800,25{,}600,51{,}200\}$.}
    \label{fig:inverse-sqrt-alpha-sensitivity-4x}
\end{figure*}

\clearpage
\subsection{Figure~\ref{fig:large_batch} at optimal hyperparameters}
We similarly provide the losses from the experiment in Figure~\ref{fig:large_batch} for the optimal hyperparameters at each intermediate point.

\begin{figure*}[!htp]
    \centering
    \includegraphics[width=1.0\linewidth]{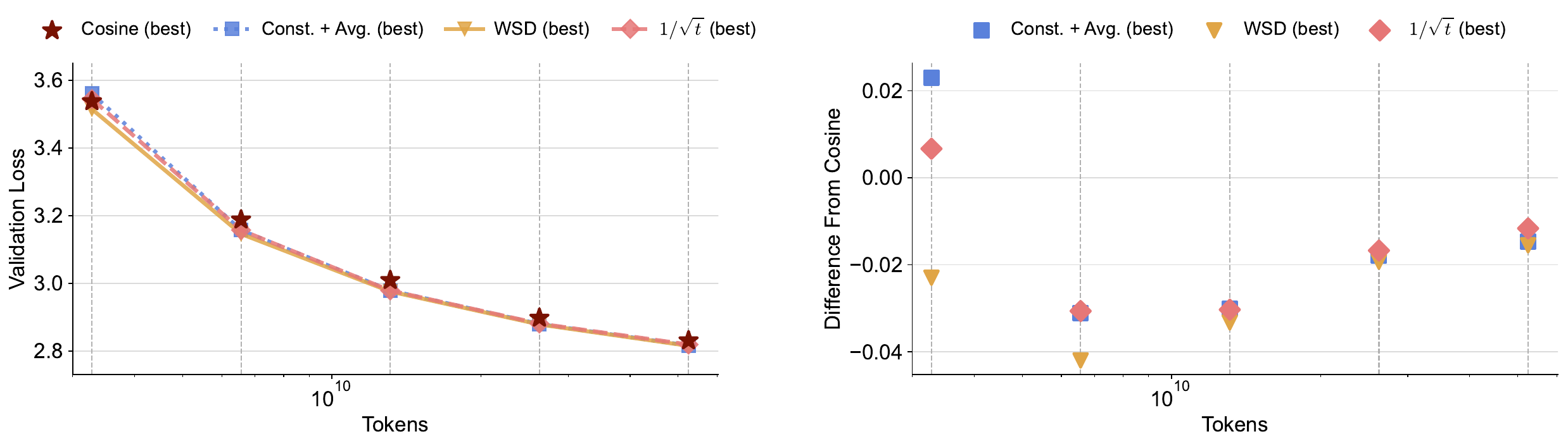}
    \caption{(Left) Validation loss comparison between cosine decay, constant learning rate with averaging, $1/\sqrt{t}$ with averaging and WSD on 150M models trained on $1 \times$ to $16\times$ Chinchilla size datasets, at batch size $4096$ with a 500-step warmup. (Right) The difference between the loss achieved by each schedule and cosine at each multiple of Chinchilla, with negative values meaning better than cosine. Traces are made by interpolating between the losses of the optimal hyperparameter runs at each of the intermediate points from $1\times$ to $16 \times$ Chinchilla.}
\label{fig:large_batch_optimal_hps}
\end{figure*}

\clearpage
\subsection{Synthetic sensitivity to label noise}
Figure~\ref{fig:noalpha_synthetic_0.01} repeats the exact-recursion experiment from Figure~\ref{fig:noalpha_synthetic_0.1} at a lower label-noise variance. This isolates the expected finite-horizon effect of the noise scale: reducing $\sigma^2$ delays the variance-dominated regime in which the faster asymptotic decay of $1/\sqrt{t}$ becomes visible.

\begin{figure*}[!htp]
    \centering
    \includegraphics[width=1.0\textwidth]{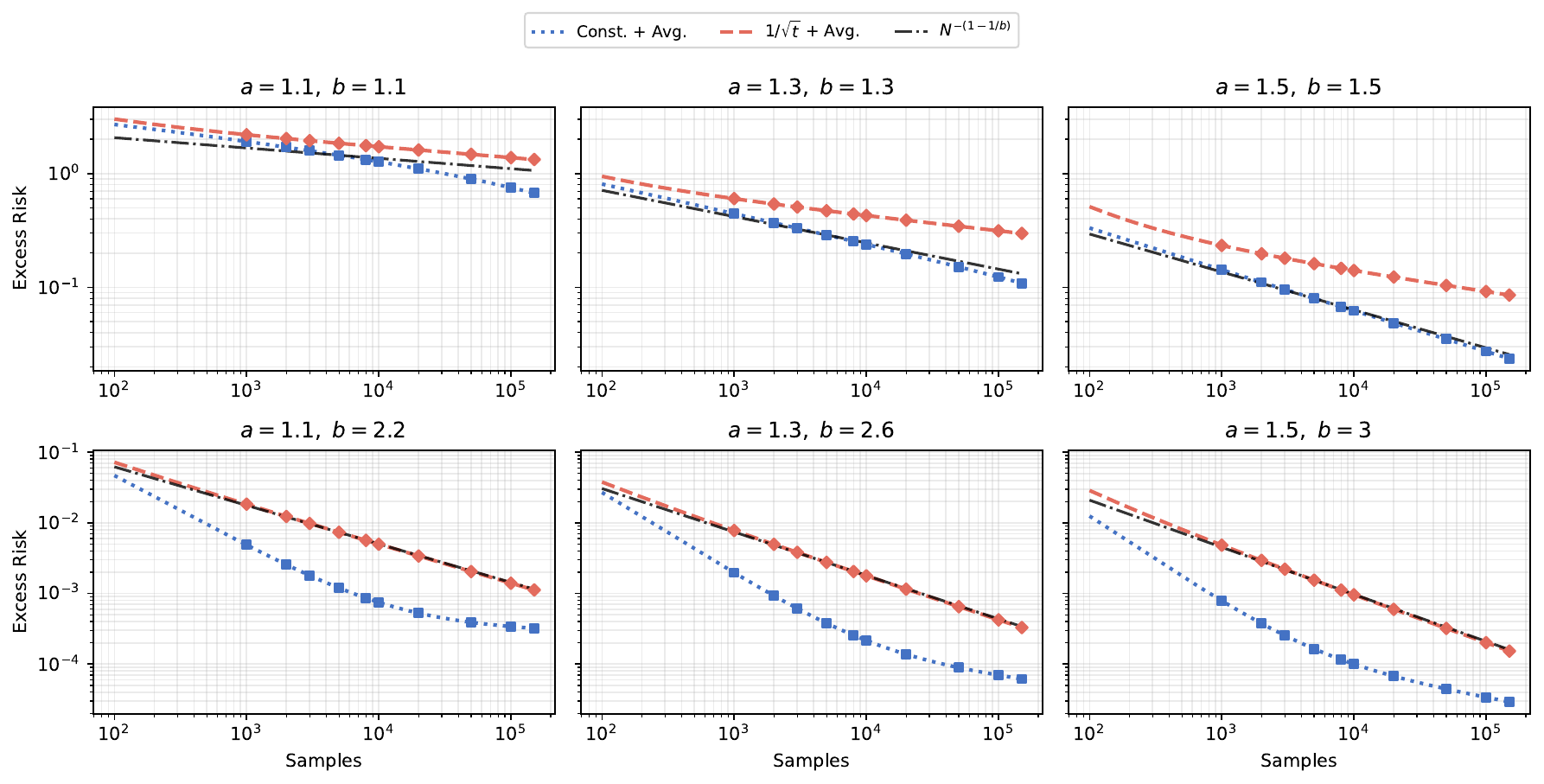}
    \caption{Similar setup as Figure~\ref{fig:noalpha_synthetic_0.1}, but with $\sigma^2 = 0.01$.}
    \label{fig:noalpha_synthetic_0.01}
\end{figure*}

\clearpage
\subsection{Optimal cosine envelope}
Figure~\ref{fig:cos-transfer} compares two ways of using cosine learning-rate schedules across token budgets from $1\times$ to $32\times$ Chinchilla.
First, we train \emph{separately tuned} cosine baselines for each budget and connect the best result at every horizon; this defines the \emph{optimal cosine envelope} (red curve).
Second, we tune a cosine schedule for a long run ($8\times$, $16\times$, or $32\times$) and evaluate it at earlier checkpoints along the same training trajectory.
Across both model sizes, these intermediate points lie substantially above the optimal envelope, showing that a cosine schedule tuned for a long horizon does not transfer well to shorter budgets.

Note that we plot again Figure~\ref{fig:cosine_envelope} to aid in side by side comparison. 

\begin{figure}[!htp]
  \centering
  \begin{subfigure}[t]{0.49\textwidth}
    \centering
    \includegraphics[width=\linewidth]{figures/cosine_transfer_left.pdf}
    \caption{150M}
    \label{fig:cos-transfer-150m}
  \end{subfigure}\hfill
  \begin{subfigure}[t]{0.49\textwidth}
    \centering
    \includegraphics[width=\linewidth]{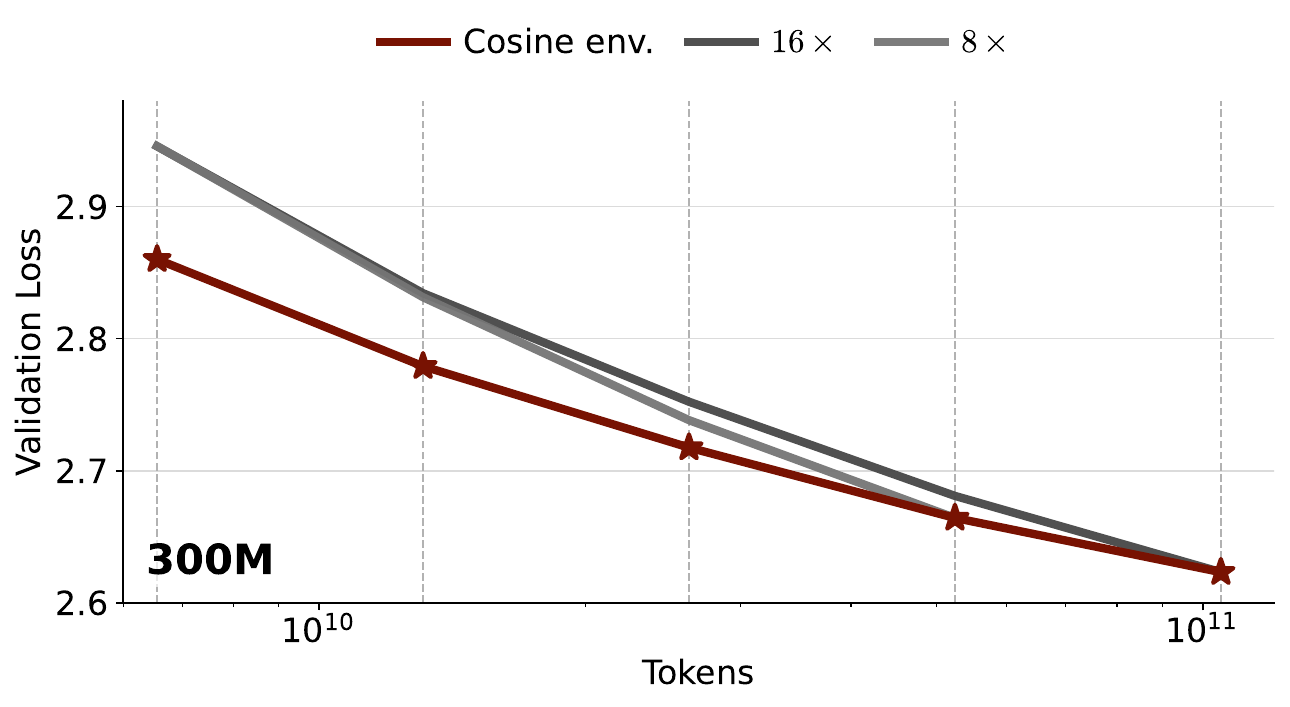}
    \caption{300M}
    \label{fig:cos-transfer-300m}
  \end{subfigure}
  \caption{
  Comparison between the optimal cosine envelope and checkpoints from cosine schedules tuned for longer budgets.
  The optimal envelope (red) is formed by independently tuning a cosine decay for each training horizon ($1\times$--$32\times$ Chinchilla) and taking the best validation value at that horizon.
  We also plot cosine schedules tuned for $8\times$, $16\times$, and $32\times$ and evaluated at smaller budgets using intermediate checkpoints from the same run.
  The gap to the envelope is substantial, indicating limited transfer from long-horizon cosine tuning to shorter budgets.
  In the 150M experiment, checkpoints were not logged exactly at $1\times$--$32\times$, and we plot the closest recorded validation point.
  }
  \label{fig:cos-transfer}
\end{figure}

\end{document}